\documentclass[twoside]{article}

\usepackage[accepted]{aistats2026}
\usepackage[round]{natbib}

\usepackage[utf8]{inputenc} 
\usepackage[T1]{fontenc}    
\usepackage{hyperref}       
\usepackage{url}            
\usepackage{booktabs}       
\usepackage{amsfonts}       
\usepackage{nicefrac}       
\usepackage{microtype}      
\usepackage{xcolor}         

\usepackage[british]{babel}
\usepackage{mathtools} 
\usepackage{booktabs} 
\usepackage{tikz} 
\usepackage{comment}

\usepackage{soul}

\newcommand{\R}{{\mathbb R}}
\newcommand{\N}{{\mathbb N}}

\newcommand{\cY}{{\cal Y}}

\newcommand{\cX}{{\cal X}}

\newcommand{\bE}{\mathbf{E}}

\newcommand{\cT}{\mathcal{T}}

\newcommand{\cG}{{\cal G}}

\newcommand{\bx}{{\bf x}}
\newcommand{\bX}{{\bf X}}
\newcommand{\bY}{{\bf Y}}
\newcommand{\bZ}{{\bf Z}}

\newcommand{\bV}{{\bf V}}

\newcommand{\bz}{{\bf z}}
\newcommand{\by}{{\bf y}}
\newcommand{\bW}{{\bf W}}
\newcommand{\bC}{{\bf C}}

\newcommand{\Adj}{\operatorname{Adj}}
\newcommand{\PA}{\operatorname{PA}}
\newcommand{\pa}{\operatorname{pa}}

\newcommand{\CI}{\operatorname{CI}}

\newcommand\independent{\protect\mathpalette{\protect\independenT}{\perp}}
\def\independenT#1#2{\mathrel{\rlap{$#1#2$}\mkern2mu{#1#2}}}
\newcommand{\ind}{\independent}

\newcommand{\MD}{\operatorname{MD}}
\DeclareMathOperator*{\argmin}{arg\,min}

\usepackage{etoolbox}

\newtoggle{printcomments}
\toggletrue{printcomments} 
\togglefalse{printcomments} 

\newcommand\dominik[1]{%
  \iftoggle{printcomments}{%
    \textcolor{cyan}{Dominik: #1}%
  }{}%
}

\newcommand\philipp[1]{%
  \iftoggle{printcomments}{%
\textcolor{red}{Philipp: #1} 
  }{}%
}

\newcommand\todo[1]{%
  \iftoggle{printcomments}{%
\textcolor{red}{#1} 
  }{}%
}

\definecolor{lightgray}{gray}{0.85}
\sethlcolor{Dandelion}

\usetikzlibrary{arrows,shapes,plotmarks,positioning, backgrounds, fit}
\usetikzlibrary{decorations.markings}

\tikzset{>=stealth'} 
\tikzset{node distance=1.5cm}
\tikzstyle{graphnode} = 
   [circle,draw=black,minimum size=22pt,text centered,text
     width=22pt,inner sep=0pt] 
\tikzstyle{var}   =[graphnode,fill=white]
\tikzstyle{vardashed}   =[graphnode,draw=gray,fill=white]
\tikzstyle{obs}   =[graphnode,fill=black,text=white]
\tikzstyle{obsgrey}   =[graphnode,draw=white,fill=lightgray,text=black]
\tikzstyle{par}    =[graphnode,draw=white,fill=red,text=black] 
 \tikzstyle{crucial} =[graphnode,draw=white,fill=yellow,text=black] 
\tikzstyle{fac}   =[rectangle,draw=black,fill=black!25,minimum size=5pt]
\tikzstyle{facprior} =[rectangle,draw=black,fill=black,text=white,minimum size=5pt]
\tikzstyle{edge}  =[draw=white,double=black,very thick,-]
\tikzstyle{blueedge}  =[draw=white,double=blue,very thick,-]
\tikzstyle{rededge}  =[draw=white,double=red,very thick,-]
\tikzstyle{prior} =[rectangle, draw=black, fill=black, minimum size=
5pt, inner sep=0pt]
\tikzstyle{dirprior} = [circle, draw=black, fill=black, minimum
size=5pt, inner sep=0pt]

\tikzstyle{dot_node}=[draw=black,fill=black,shape=circle]

\usepackage{csquotes}
\usepackage[shortlabels]{enumitem}
\setlist[itemize]{noitemsep,leftmargin=*,nosep}
\setlist[enumerate]{noitemsep,leftmargin=*,nosep}
\usepackage{cleveref}
\usepackage[standard]{ntheorem}
\newtheorem{assumption}{Assumption}

\usepackage{algorithm}
\usepackage{algorithmic}
\usepackage{subfig}
\definecolor{tabgreen}{HTML}{2ca02c}
\definecolor{tabred}{HTML}{d62728}
\definecolor{tabblue}{HTML}{1f77b4}
\definecolor{taborange}{HTML}{ff7f0e}
  
\begin{document}

%

%

\twocolumn[

\aistatstitle{On Different Notions of Redundancy in Conditional-Independence-Based Discovery of Graphical Models}

\aistatsauthor{ Philipp M. Faller \And Dominik Janzing$^{*}$}

\aistatsaddress{ Karlsruhe Institute of Technology\\
	              Karlsruhe, Germany\\ 
                \texttt{philipp.faller@partner.kit.edu} 
            \And  Amazon Research\\
	            Tübingen, Germany\\
                \texttt{janzind@amazon.com}} ]

\begin{abstract}
Conditional-independence-based discovery uses statistical tests to identify a graphical model that represents the independence structure of variables in a dataset.
These tests, however, can be unreliable, and algorithms are sensitive to errors and violated assumptions.
 Often, there are tests that were not used in the construction of the graph.
In this work, we show that these \emph{redundant} tests have the potential to \emph{detect} or sometimes \emph{correct} errors in the learned model.
But we further show that not all tests contain this additional information and that such redundant tests have to be applied with care.
Precisely, we argue that the conditional (in)dependence statements that hold for every probability distribution are unlikely to detect and correct errors---in contrast to those that follow only from graphical assumptions.
\end{abstract}

\section{INTRODUCTION}
Graphical models have become an indispensable tool for understanding complex systems and making informed decisions in various scientific disciplines \citep{lauritzen1996graphical}. 
They provide insights into the structure within a system, and under some additional assumptions, they can be interpreted as \emph{causal models} \citep{pearl2009causality,spirtes2000causation}.

Conditional independence (CI) statements are utilised to infer the graphical structure by algorithms such as PC \citep{spirtes2000causation} or SP \citep{raskutti2018learning}.
However, a key challenge arises from the statistical hardness of conditional independence tests. 
As shown by \citet{shah2020hardness}, CI-tests cannot have a valid false positive control and power against arbitrary alternatives simultaneously. \todo{double check if this is what the paper says}
Additionally, constraint-based algorithms often rely on assumptions like \emph{faithfulness}, which means that a graph not only implies \emph{in}dependences but also dependences.
\citet{uhler2013geometry} showed that this assumption can be problematic in the finite-sample regime since even faithful distributions can be close enough to unfaithful ones for a CI-test to fail.
It is also violated whenever systems evolve to equilibrium states, as in many biological settings \citep{andersen2013expect}.
On the other hand, even a single wrong result of a CI-test can result in arbitrarily large changes in the resulting graphical model (as shown in \cref{ex:wrong_v_structure} in \cref{sec:additional_examples}).

In the worst case, CI-based discovery of graphical models requires exponentially many CI-tests in the number of nodes \citep{korhonen2024structural,zhang2024membership}.
Despite the large number of required tests, the set of possible graphical models can still be small compared to the set of possible combinations of CI-statements.
While e.g. \citet{shiragur2024causal} proposed a method to reduce the number of required tests (while sacrificing details of the models),
in this paper we will advocate for using additional CI-tests to evaluate the graphs, which has been proposed implicitly or explicitly before  \citep{textor2016robust,eulig2025toward,janzing2023reinterpreting}.
In spirit, this follows \citet{raskutti2018learning}, who have hypothesised that there exists a statistical/computational trade-off for causal discovery. 
We will argue that not all CI-tests carry much additional information (and can therefore be used to evaluate a graphical model), but only those tests for which the result follows from \emph{graphical} restrictions instead of the laws of probability. 

\begin{example}[non-generic collider]
\label{ex:motivation}
	Consider a probability distribution that is Markovian and faithful to the graph $X_1\to Y\leftarrow X_2$ for random variables $X_1, X_2, Y$.
	Suppose we use the PC algorithm to recover the graph.
	The algorithm will conduct all pairwise marginal independence tests.
	These tests already identify the given DAG.
    But clearly, the graph also entails $X_1\not\ind X_2\mid Y$  under the faithfulness assumption.
	On the other hand, this dependence does not follow for all probability distributions,
    as we will see by constructing a counterexample. 
    We will use similar constructions multiple times.
	Assume\footnote{
    One could also think about a variable $Y$ taking values in the natural numbers, where disjoint sets of bits in the binary expansion of $Y$ depend on $X_1$ and $X_2$, respectively.} $Y=(Y_1, Y_2)\in \R^2$. Now assume $Y_1$ only depends on $X_1$ and $Y_2$ only on $X_2$ as in \cref{fig:three_v_structures}.
	Then we would get the same marginal (in)dependences as before, but not the conditional dependence from above.
     On the other hand, if we have $X_1\not\ind Y$ and $X_1\ind X_2$ we also have $X_1
    \not\ind Y\mid X_2$ for all probability distributions (see \cref{sec:small_claims} for a derivation).
    This follows from the Graphoid axioms that we will refer to throughout the paper (see \cref{def:graphoid}).
    In other words, there is a dependence, $X_1\not\ind X_2\mid Y$, that follows from the assumption that the underlying distribution can be represented by a \emph{faithful} DAG, but the dependence does not hold for all distributions.
At the same time, there is a CI-statement that carries no additional information, namely $X_1
    \not\ind Y\mid X_2$. \todo{Set tikz nodes size back to 22pt}

We can use tests like $X_1\ind X_2\mid Y$ in two ways: either to make the result of graph discovery more robust, or for evaluation---similar to held-out samples in statistical cross-validation \citep{bishop2006pattern} or additional bits in an error-detecting code (see also \cref{subsec:noisy_channel}).
Hence, we will call them \emph{redundant}.
But we will see throughout the paper that these redundant tests can detect and correct (among others) errors from faithfulness violations, which is impossible with methods like statistical cross-validation.

\paragraph{Contributions}
This paper aims to provide a novel perspective on CI-based discovery of graphical models. Precisely,
\begin{itemize}
	\item to the best of our knowledge, we are the first to point out that the dependence between CI-statements impacts which tests should be used in graph discovery, 
	\item we show  that tests which already follow from previous ones in all distributions can give a misleading impression of evidence for a graphical model, while the ones that only follow through graphical assumptions are more likely to falsify the model,
	\item we show why the tests that follow for all distributions cannot be used to correct errors, while the ones that follow from graphical assumptions alone can, 
    	\item we show how our novel perspective generalises previous results on the robustness of graph discovery.
\end{itemize}
In other words, 
we are the first to systematically investigate and categorise the redundancy of 
CI-statements, and particularly to use these notions to evaluate and 
improve graphs with CI-tests that are \enquote{held-out} in the sense that they 
do not follow from the tests used for the graph discovery.
This work aims to contribute to the discussion on how graphical models should be evaluated and to question which empirical observations are real evidence and thus capable of corroborating a model.

\begin{figure}[t]
\centering
    \begin{tikzpicture}
	\node[obs] (A) {$X_1$};
	\node[obs, right=of A] (B1) {$Y_1$};
	\node[obs, right=of A, xshift=1cm] (B2) {$Y_2$};
	
	\begin{scope}[on background layer]
		\node[draw, ellipse, dashed, fit=(B1) (B2), inner sep=0.05cm] (B_group) {};
	\end{scope}
	
	\node[obs, right=of B2] (C) {$X_2$} edge[->] (B2);
	
	\draw[->] (A) -- (B1);

\end{tikzpicture}
	\caption{The marginal independence tests identify the faithful DAG (with $Y$ as a single variable). But the collider structure $X_1\to Y\leftarrow X_2$ implies $X_1\not\ind X_2 \mid Y$ for all \emph{faithful} distributions. This does not hold for every distribution. On the contrary, given the marginal tests we have, e.g., $X_1\not \ind Y\mid X_2$ for all distributions.}
	\label{fig:three_v_structures}
\end{figure}
\end{example}

\section{REDUNDANCY OF CIs}
\paragraph{Notation}
We will now introduce some notation and basic concepts.
For more detailed definitions, we refer the reader to \cref{sec:further_definitions}.
We denote a random variable with an upper case letter $X$. 
A set of random variables is denoted with boldface letter $\bX$.
Let $\bV$ be a finite set of variables.
An independence model $M$ over $\bV$ is a set of triplets $(\bX, \bY, \bZ)$, where $\bX, \bY, \bZ\subseteq \bV$, $\bX\neq\emptyset \neq\bY$ and $\bX, \bY, \bZ$ are disjoint.
\dominik{a model does not come with indep/dep information - it is just a triple?} \philipp{The triplets define the independence and dependence is defined implicitly as the complement.}
We say $\bX$ is \emph{independent} from $\bY$ given $\bZ$ and write $\bX\ind_{\mspace{-8mu} M}\bY\mid \bZ$, where we often omit the subscript.
If $(\bX, \bY, \bZ) \not\in M$ we say they are dependent and write  $\bX\not\ind_{\mspace{-8mu} M}\bY\mid \bZ$.
A \emph{CI-statement} is a quadruple of $\bX, \bY, \bZ$ and a boolean value, indicating whether the independence holds.
For a set of CI-statements $L$ we slightly abuse notation and write $L\subseteq M$ if for $(\bX, \bY, \bZ, b)\in L$ we have $b=((\bX,\bY,\bZ)\in M$).
We also sometimes write a CI-statement as function ${\CI: \bX, \bY, \bZ\mapsto (\bX, \bY, \bZ, b)}$. 
Note that with \emph{in}dependences we refer to statements of the form $\bX\ind \bY\mid \bZ$ and with dependences to $\bX\not\ind \bY\mid \bZ$, while with CI-statement we refer to both of them.
A probability distribution over $\bV$ induces an independence model via probabilistic conditional independence.
Graphical models can represent independence models, and we denote a model induced by a graph $G$ as $M_G$.
For an undirected graph $G$ we define $(\bX, \bY, \bZ) \in M_G$ iff $\bX$ is separated from $\bY$ given $\bZ$.
For DAGs $(\bX, \bY, \bZ) \in M_G$ iff $\bX$ is $d$-separated from $\bY$ given $\bZ$.
In both cases, we also write $\bX\perp_G \bY\mid \bZ$.
By \emph{graphical model} we refer to \emph{either} an undirected graph or a DAG (and its respective independence model).\footnote{
Our insights can be applied to any model with a notion of independence between nodes, such as chain graphs \\ \citep{lauritzen1996graphical}, completed partial DAGs, maximal ancestral graphs, partial ancestral graphs \citep{spirtes2000causation}, or acyclic directed mixed graphs \citep{richardson2003markov}.} \todo{remove line break}
A graph $G$ is \emph{Markovian} to an independence model $M$ if $(\bX, \bY, \bZ) \in M_G \implies(\bX, \bY, \bZ) \in M$. It is \emph{faithful} if $(\bX, \bY, \bZ) \in M \implies (\bX, \bY, \bZ) \in M_G$.
Again, we abuse notation and say $G$ is Markovian to a set of CI-statements $L$ when all \emph{in}dependences in $G$ are contained and true in $L$.
As shorthand, we use 
 $   \CI(\bV) := \{(X, Y, \bZ) :\: X, Y\in \bV,\, \bZ\subseteq \bV\setminus\{X, Y\},\, X\neq Y\}$ for triplets of nodes that often occur in our CI-statements.


\subsection{Graphoids, graphs and redundancy}
The central observation of this paper is that a graphical model usually entails more CI-statements than the ones necessary to identify its Markov equivalence class, as in  \cref{ex:motivation}.
In other words, the space of independence models entailed by a graphical model is typically smaller than the space of possible independence models.
The additional tests are our candidates to be used for error detection and correction, which motivates the following definition.

\begin{definition}[GM-redundancy]
	\label{def:graphical_redundancy}
	Let $L$ be a set of CI-statements  and $s\not\in L$ be another CI-statement.
	Let $\cG$ be a set of graphical models. 
	We call $s$ \emph{graphical-model-redundant} (GM-redundant) w.r.t. $\cG$ and $L$ if $\{s\}\subseteq M_G$ whenever $L\subseteq M_G$ for any graph $G\in \cG$.
\end{definition}
\todo{In camera-ready write out wrt.}
Note that this definition is with respect to a set of (previous) CI-tests \emph{and} a class of graphical models.
The former depends on the specific algorithm used for discovery, while the latter depends on the assumptions that we make about the data. \dominik{these remarks before and after are a bit confusing, I would be more explicit by saying that ${\cal G}$ is the set of graphs consistent with the tests made, subject to faithfulness.}
In \cref{ex:motivation}, both $X_1\not\ind X_2\mid Y$ and $X_1\not\ind Y\mid X_2$ are GM-redundant with respect to the marginal independence tests, as they hold in every DAG that represents the given marginal CI-statements faithfully.

\label{subsec:redundancy}

In coding theory, one can often assume that bit errors occur independently.
But it is well-known that the output of CI-tests on the same dataset can be dependent.
In particular, there are cases where some (in)dependence statements follow from a set of (in)dependence statements for \emph{all} probability distributions.
\citet{pearl2022graphoids} have provided a sound set of logical rules to derive independences: the (semi) Graphoid axioms.
\begin{definition}[Graphoid Axioms]
	\label{def:graphoid}
	Let $M$ be an independence model over variables $\bV$ and $\bX, \bY, \bZ, \bW\subseteq \bV$ be disjoint with $\bX\neq\emptyset\neq \bY$. We call $M$ a \emph{semi-Graphoid} if the following properties hold
	\begin{enumerate}
		\item Symmetry: $\bX\ind \bY\mid \bZ \iff \bY\ind \bX\mid  \bZ$
		\item Decomposition: $\bX\ind  \bY\cup \bW \mid \bZ \\\implies \bX \ind \bY \mid\bZ \ \land\  \bX \ind \bW \mid \bZ$
		\item Weak Union: $\bX\ind \bY\cup \bW \mid \bZ  \\\implies \bX\ind  \bY \mid\bZ\cup \bW\ \land\  \bX\ind\bW \mid \bZ\cup \bY$
		\item Contraction: $\bX\ind\bY \mid \bZ \ \land\  \bX \ind\bW \mid \bZ\cup \bY \\\implies \bX\ind\bY\cup \bW \mid \bZ $
	\end{enumerate}
	$M$ is called \emph{Graphoid} if we further have
	\begin{enumerate}[resume]
		\item Intersection: $\bX\ind\bY \mid \bZ\cup \bW \ \land\  \bX\ind\bW\mid \bZ\cup \bY \\\implies \bX\ind\bY\cup \bW\mid \bZ$.
	\end{enumerate}
\end{definition}

The independence model of a distribution is a semi-Graphoid, and it is a Graphoid  if the distribution is positive \citep{lauritzen1996graphical}.
Although these rules are sound, they are not complete, and there cannot be a finite, sound, and complete set of axioms to describe conditional independence in probability distributions \citep{studeny1992conditional}.
Since the semi-Graphoid rules hold for every distribution, they also hold for every empirical distribution.
Moreover, the following proposition shows that for conditional mutual information (which is zero if and only if \emph{in}dependence holds),  the Graphoid axioms have a certain continuity property.

\begin{figure}[t]
    \centering
    \begin{tikzpicture}
      
        \begin{scope}
            \node[ellipse, draw, thick, fill=tabblue!30, minimum width=0.95\columnwidth, minimum height=6\baselineskip] (graphically) {};
        \end{scope}

        \begin{scope}[shift={(0,-0.1\baselineskip)}]
            \node[ellipse, draw, thick, fill=taborange!30, minimum width=0.9\columnwidth, minimum height=4\baselineskip] (probabilistic) {};
        \end{scope}
        
        \begin{scope}[shift={(0,-0.2\baselineskip)}]
            \node[ellipse, draw, thick, fill=tabgreen!30, minimum width=0.85\columnwidth, minimum height=2\baselineskip] (graphoid) {};
        \end{scope}

        \node at (graphoid) {GA-Redundant};
        \node[below] at (graphically.north) {GM-Redundant};
        \node[below] at (probabilistic.north) {P-Redundant};
    \end{tikzpicture}
    \caption{Hierarchy of \cref{def:graphical_redundancy,def:graphoid_redundancy,def:prob_redundant}. We argue to use GM- but not P-redundant CI-statements.}
\end{figure}

\begin{proposition}[continuity of a Graphoid]
    \label{prop:continuity_mi}
    Let $\bX, \bY, \bZ, \bW\subseteq \bV$ be disjoint and $\epsilon, \delta > 0$.
    Then
    \begin{enumerate}
		\item $I(\bX:\bY\mid \bZ) \le \epsilon \iff I(\bY:\bX\mid \bZ)\le \epsilon$
		\item $I(\bX:\bY\cup \bW\mid \bZ) \le \epsilon \\\implies  I(\bX:\bY\mid \bZ) \le \epsilon \ \land\  I(\bX:\bW\mid  \bZ) \le \epsilon$
		\item $I(\bX:\bY\cup \bW\mid \bZ) \le \epsilon \\ \implies I(\bX:\bY\mid \bZ\cup \bW) \le \epsilon \ \land\  I(\bX:\bW\mid \bZ\cup \bY) \le \epsilon$
		\item $I(\bX:\bY\mid \bZ) \le \epsilon \ \land\  I(\bX:\bW\mid \bZ\cup \bY) \le \epsilon \\ \implies I(\bX:\bY\cup \bW\mid \bZ) \le 2\epsilon$
	\end{enumerate} 
    If we further assume $I(\bW:\bY \mid \bZ) \le \delta$ we get
	\begin{enumerate}[resume]
		\item $I(\bX:\bY \mid \bZ\cup \bW) \le \epsilon \land I(\bX:\bW \mid \bZ\cup \bY) \le \epsilon \\\implies I(\bX : \bY\cup \bW \mid \bZ) \le \delta + 2\epsilon.$
	\end{enumerate}
\end{proposition}
All proofs are in \cref{sec:proofs}. 
This means that even if some CI-statements hold only approximately, they are likely to influence other test results according to the Graphoid axioms.
In other words, even a very weak dependence that is not detected by a test still influences other tests almost as if there were no dependence.
Accordingly, the influenced tests contain less (or no) redundant information in the sense that we are interested in.
Therefore, we want to differentiate between tests that are implied by the graph alone and tests that already follow for all probability distributions.

\begin{definition}[P-redundancy]
	\label{def:prob_redundant}
	Let $L$ be a set of CI-statements  and $s\not\in L$ be another CI-statement.
	We call $s$ \emph{probabilistically redundant} (P-redundant) with respect to $L$ if $\{s\}\subseteq M$ for any independence model induced by a probability distribution with $L\subseteq M$. 
\end{definition}
The CI-statements we are interested in are the ones that are not P-redundant. Yet, it is hard to operationalise this definition.
Although \citet{niepert2009logical} shows that the problem of whether a CI-statement follows from a given set of other CI-statements is decidable for variables with finite support, 
to the best of our knowledge, there are no results on decidability for continuous variables.
To render the problem decidable in any case, we will restrict ourselves to CI-statements that follow via the Graphoid axioms.
\begin{definition}[GA-redundancy]
	\label{def:graphoid_redundancy}
	Let $L$ be a set of CI-statements  and $s\not\in L$ be another CI-statement.
	We call $s$ \emph{Graphoid-axiom-redundant} (GA-redundant) with respect to $L$ if $\{s\}\subseteq M$ for any Graphoid independence model with $L\subseteq M$. 
\end{definition}

Since the Graphoid axioms are sound, GA-redundancy\footnote{For simplicity, we do not further distinguish between semi-Graphoid-axiom-redundant and Graphoid-axiom-redundant CI-statements.} is a sufficient condition for a CI-statement to also be P-redundant.
In \cref{ex:prob_but_not_graphoid} in \cref{sec:additional_examples}, we show CI-statements that are P-redundant but not GA-redundant.

The following definition captures the CI-statements that are not already implied by the Graphoid axioms but follow solely from graphical assumptions.
\begin{definition}[PGM-redundancy]
	\label{def:purely_graphical_red}
	Let $L$ be a set of CI-statements  and $s\not\in L$ be another CI-statement.
	We call $s$ \emph{purely graphical-model-redundant} (PGM-redundant) with respect to $L$ if $s$ is GM-redundant but not GA-redundant. 
\end{definition}
\dominik{one could argue against this definition because a P-redundant but not GA-redundant CI is not 'purely' graphically redundant.}
In \cref{ex:motivation}, only $X_1\not\ind X_2\mid Y$ is PGM-redundant with respect to the marginal CI-statements. 

\subsection{Graphical Criteria for Redundancy}
These definitions beg the question of whether there is a criterion to find the PGM-redundant CI-statements.
We will first present results showing which CI-statements cannot be PGM-redundant, and then see a sufficient graphical criterion that covers a broad class of examples where they are.
In particular, \cref{prop:markov_network_markovian,prop:protocol_graph_markovian} show cases where all \emph{in}dependences are GA-redundant.
\begin{corollary}[\citet{verma1990causal} Thm. 2]
	\label{prop:protocol_graph_markovian}
	Let $M$ be a Graphoid independence model 
    over a set of nodes $\bV$, and $\pi:\bV\to \N$ be an ordering of $\bV$.
	Let $L_\pi$ be a list of CI-statements from $M$ such that  for $X, Y, \bZ \in \CI(\bV)$ we have $\CI(X, Y\mid \bZ) \in L_\pi$ iff $\pi(X) < \pi(Y)$ and
	\begin{displaymath}
		\bZ = \{Z\in\bV :\: \pi(X) \neq \pi(Z) < \pi(Y)\}. 
	\end{displaymath}
	Let $G(\pi)$ be the DAG over $\bV$ with an edge from $X\in\bV$ to $Y\in \bV$ iff $(X\not\ind Y \mid \bZ) \in L_\pi$ for some $\bZ\subseteq \bV$.
    Then every \emph{in}dependence in $M_{G(\pi)}$ is GA-redundant with respect to $L_\pi$.
\end{corollary}
\begin{corollary}[\citet{geiger1993logical} Thm. 12]
	\label{prop:markov_network_markovian}
	Let $M$ be a Graphoid independence model 
    over a set of nodes $\bV$.
	Let $L$ be the set of CI-statements
	\begin{displaymath}
		L = \{\CI(X, Y\,|\, \bV\setminus\{X, Y\}) :\: X, Y\in \bV,\, X\neq Y\}
    \end{displaymath}
    from $M$.
	Let $G(L)$ be the undirected graph over $\bV$ with an edge between $X, Y\in \bV$ iff $(X\not\ind Y \mid \bV\setminus\{X, Y\}) \in L$.
	Then every \emph{in}dependence in $M_{G(L)}$ is GA-redundant with respect to $L$.
\end{corollary}

No graph discovery algorithm can rely on the Markov assumption alone.
Most of them also use the faithfulness assumption. 
The latter is especially troublesome, as there are many applications where it may be violated, as we have mentioned before.
\dominik{do we have references there?}
But in the cases above, none of the \emph{in}dependence statements can be used as additional redundancy.
So here, it is only due to faithfulness (and similar assumptions) that we can have PGM-redundancy at all.
But also, not all of the dependences are good candidates for error detection and correction.
As \citet{bouckaert1995bayesian} shows, there are also dependences that follow from the Graphoid axioms as their contrapositives.
In \cref{prop:coupling} in \cref{sec:graphical_crit_dependences}, we show how these insights can be applied in the situations of \cref{prop:markov_network_markovian,prop:protocol_graph_markovian}.

On the other hand, we will now give a criterion under which a  dependence statement is guaranteed to be PGM-redundant.
So whenever a graphical model implies such a dependence,  it is a good candidate for additional redundancy.

\begin{figure*}[t]
    \centering
    \begin{tikzpicture}
        \node[obs] (A) {$A$} ;
        \node[obs, right=of A] (X) {$X$} edge[-] (A) ;
        \node[obs, right=of X] (V) {$Z$} edge[-] (X) ;
        \node[obs, right=of V] (Y) {$Y$} edge[-] (V) ;
        \node[obs, right=of Y] (B) {$B$} edge[-] (Y) ;
        \draw (X) edge[-, bend left=20] (Y); 
        
    \end{tikzpicture}
    \caption{$A$ and $B$ are \emph{coupled over} $(X, Y, \emptyset)$ given $\emptyset$ but are not coupled over $(X, Y, Z)$ given $\emptyset$, since the path $X-Z-Y$ stays active given $\emptyset$ if we \enquote{deactivate} all paths between $X$ and $Y$ that are active given $Z$.}
    \label{fig:s_coupling_example}
\end{figure*}

\begin{definition}[coupling over nodes]
\label{def:s_coupling}
    Let  $(X, Y, \bZ)$, $(A, B, \bC) \in \CI(\bV)$, $G$ be a graphical model over $\bV$, and $s:=(X, Y, \bZ)$.
We say a path $(A, \dots, B)$ is \emph{$s$-active} given $\bC$ if 
\begin{enumerate}
    \item 
    it is active given $\bC$ (according to the respective graphical separation) and
    \item  
    there is no sub-path $(X, \dots, Y)$ that is active given $\bZ$.
\end{enumerate}
Further, we say $A$ and $B$ are \emph{coupled over} $s$ given $\bC$ iff there is at least one active but no $s$-active path
between $A$ and $B$ given $\bC$.
\end{definition}
Intuitively, an $s$-active path stays active if we \enquote{deactivate} all paths between $X$ and $Y$, and $A, B$ being coupled over $s$ means that all paths between $A$ and $B$ are \enquote{mediated} by $X$ and $Y$.
See \cref{fig:s_coupling_example} for an example.
\begin{proposition}[sufficient criterion] 
	\label{prop:sufficient_cond_non_graphoid}
	Let $M$ be the independence model of a positive distribution, $L$ be a set of CI-statements from $M$ and $s:=(X, Y, \bZ) \in \CI(\bV)$ with $(X\not\ind Y\mid \bZ)\not\in L$.
	Let $G$ be a graphical model 
    such that $G$ is Markovian to $L$ and $X\not\perp_{G} Y\mid \bZ$.
	If there is no $(A \not\ind B\mid \bC) \in L$ such that $A$ and $B$ are coupled over $s$ given $\bC$, then $X\not\ind Y\mid \bZ$ is PGM-redundant given $L$.
\end{proposition}

Note that the PC algorithm is not guaranteed to return a graph that is Markovian to the conducted tests if its assumptions are violated (see \cref{ex:pc_not_markov} in \cref{sec:additional_examples}) and therefore \cref{prop:sufficient_cond_non_graphoid} cannot always be applied there.

\dominik{Isn't this a crucial result? It is presented as a minor one.}

\todo{Remove? 
\begin{remark}
Given that the PGM-redundant tests depend on the ground truth graphical model, one might wonder whether one can define a \emph{condition number} of a graph in analogy to the condition number of a matrix, where the numerical stability of matrix inversion depends on the matrix itself. But clearly here things are even more involved as also the given CI-statements (i.e. the discovery algorithm that already conducted tests) is relevant.
\end{remark}}

\todo{Necessary?}
\todo{Put back in bridge sentence?}
\section{ERROR DETECTION AND CORRECTION}
\label{sec:error_detection_correction}

Since GA-redundant tests are unlikely to contradict the given tests, they are unlikely to reveal errors in the given tests (which does not mean that such rare contradictions are uninformative). Moreover, the absence of contradictions might even give the misleading impression of evidence for a model \dominik{and, for instance, way suggest evidence for a causal DAG although there isn't},
as we will discuss more formally in \cref{ex:graphoid_hurts_correction}.
Therefore, we will \emph{detect} and \emph{correct} errors with PGM-redundant tests in the following.
\todo{Put back sentence to end of previous section?}
\subsection{Error detection}

At the beginning of this section, we will focus on spanning trees\footnote{A spanning tree is an undirected graph, in which every pair of nodes is connected by exactly one path.} for two reasons. First, they are a simple class of examples. But also, as we will see, they impose particularly strong graphical assumptions.

\paragraph{Error detection in spanning trees}

Suppose we use the CI-statements from \cref{prop:markov_network_markovian} to find an undirected graph.
Although this approach is quite efficient, it is sensitive to errors.
Each of the conducted CI-tests directly corresponds to the presence or absence of an edge.
In other words, a single error already makes it impossible to find the correct graph.
But, in this case, \cref{prop:sufficient_cond_non_graphoid} equips us with plenty of CI-statements that could be used to falsify the result under the faithfulness assumption.
The following corollaries show general cases where \cref{prop:sufficient_cond_non_graphoid} can be applied.
\begin{corollary}[transitive dependence]
\label{ex:transitive_dependence}
    Consider the independence model that is Markovian and faithful to the undirected graph in \cref{fig:transitive_dependence}.
    Let $L$ be the set of CI-statements from \cref{prop:markov_network_markovian}.
    Then every dependence between non-adjacent nodes is PGM-redundant, and all \emph{in}dependences are GA-redundant.

\end{corollary}
This can be seen as follows: by construction, $L$ does not contain a dependence statement involving non-adjacent nodes $X_i$ and $X_j$ where $i, j\in\N$.
In particular, $L$ does not contain a dependence between $X_k$ and $X_l$ with $k, l\in\N$ such that $X_k$ and $X_l$ are only connected over a path with more than one edge.
By \cref{prop:sufficient_cond_non_graphoid}, every dependence between non-adjacent nodes is PGM-redundant.\footnote{In practice, 
especially the dependences over short paths are of interest. As the data processing inequality \citep{mackay2003information} bounds the mutual information along a path, the dependences become weaker the longer a (non-deterministic) path is. 
}
The \emph{in}dependences follow via \cref{prop:markov_network_markovian}.
The next corollary shows that the PGM-redundant statements are imposed by \emph{strong} graphical assumptions.
\begin{corollary}[implied connectedness]
\label{ex:implied_connectedness}
    Suppose again our independence model $M$ is Markovian and faithful to the graph in \cref{fig:transitive_dependence}.
    Now let \begin{displaymath}
        L = \{(X_i\ind X_n \,|\, \bV\setminus\{X_i, X_n\}) :\: i \in \{1, \dots, n-2\}\}
    \end{displaymath}
    for $n\in\N_{>2}$.
    If $\cG$ is the set of spanning trees over $\bV$, the statement $X_{n-1} \not\ind X_n\mid \bV\setminus\{X_{n-1}, X_n\}$ is PGM-redundant.
 
\end{corollary}
Would $X_n$ be disconnected from the graph we could still get $X_{n-1} \perp X_n \mid \bV \setminus \{X_{n-1}, X_n\}$. But this is not possible in a spanning tree. So here, the PGM-redundancy comes from the spanning tree property.
But if we only assume the underlying graph is any undirected graph, this would not follow, as $X_n$ could indeed be disconnected.
In terms of error detection, this means that a test with result $X_{n-1}\ind X_n\mid \bV\setminus\{X_{n-1}, X_n\}$ would indicate that either this test, one of the tests in $L$, or our graphical assumptions are wrong.

\todo{Add example with PC-like tests?}

\paragraph{Error detection in DAGs}
First note that in DAGs we can also have PGM-redundancies along (now directed) paths as in \cref{ex:transitive_dependence}. 
But just as in \cref{ex:implied_connectedness}, there are also dependences that follow from the particular graphical assumption.
To see this, consider the following variant of \cref{fig:three_v_structures}.
\begin{corollary}[multiple colliders]
    Consider an independence model that is Markovian and faithful to the DAG with edges $X_i\to Y$ for $i=1, \dots, n\in\N_{>1}$.
    Assume we learn a DAG using a given topological ordering and the CI-statements from \cref{prop:protocol_graph_markovian}.
    Then all tests $X_i\not\ind X_j\mid Y$ with $i, j\in \{1, \dots, n\}, i\neq j$ are PGM-redundant. 
    Conversely, all  \emph{in}dependences 
    and all statements $X_i\not\ind Y\mid \{X_1, \dots, X_n\}\setminus\{X_i\}$ for $i=1, \dots, n$ 
    are GA-redundant.
\end{corollary}

    \begin{figure}[t]
		\centering
		\begin{tikzpicture}
			\node[obs] (X1) {$X_1$} ;
			\node[obs, right=of X1] (X2) {$X_2$} edge[-] (X1) ;
            \node[right=of X2] (X3) {$\dots$} edge[-] (X2) ;
            \node[obs, right=of X3] (Xn) {$X_n$} edge[-] (X3) ;
			
		\end{tikzpicture}
		\caption{In this graph, every dependence along more than one edge is PGM-redundant given the CI-statements from \cref{prop:markov_network_markovian}.}
		\label{fig:transitive_dependence}
	\end{figure}   

\looseness=-1Although \citet{ramsey2006adjacency} do not study the problem through the lens of redundancy, they observe that there can be several conditional (in)dependences that characterise a collider in a DAG.
They propose to let their \emph{Conservative PC} algorithm (CPC) indicate when these CI-statements contradict each other.\todo{get back footnote for unshielded triplet}
Precisely, suppose PC finds a skeleton $H$ that contains an unshielded triplet 
$A-B-C$.
They then consider all subsets of  the neighbours of $A$ and $C$ as potential conditioning sets, i.e. $\bZ\subseteq \Adj_H(A)$ or $\bZ\subseteq\Adj_H(C)$.
If the observed independence model is Markovian and faithful to a DAG, either all or none of the sets $\bZ$ with $A\ind C\mid \bZ$ contain $B$.
Indeed, we can phrase this in our framework as the CI-statements $A\not\ind C\mid \bZ\cup\{B\}$ being GM-redundant for the hypothesis that the underlying graph contains the collider $A\to B\leftarrow C$ or $A \ind C \mid \bZ \setminus\{B\}$ otherwise.
Further, the following \cref{ex:cpc_not_graphoid_redundant} shows that  these tests can indeed be GM-redundant. 
\begin{example}[CPC and GA-redundancy]
	\label{ex:cpc_not_graphoid_redundant}
    Consider again the graph given in \cref{fig:three_v_structures} and 
    interpret $Y_1$ and $Y_2$ as the components of a vector-valued 
    $Y$.
	Note that any independence model that is Markovian and faithful to this graph entails $X_1\not\ind Y$, $Y\not\ind X_2$, and $X_1\ind X_2$.
	Further, we have the \emph{in}dependence $X_1\ind X_2\mid Y$.
	Since $Y$ is neither in all sets that separate $X_1$ and $X_2$, nor in none of them, CPC would output a contradiction.
	But this example further shows that such a model exists and thus none of the CI-statements is GA-redundant given the others.	\todo{can I just generalise this by adding a conditioning set?}
\end{example}
Note that our perspective also includes CI-statements that cannot be detected by CPC (like nodes connected along longer paths in \cref{ex:transitive_dependence}), which shows that our work generalises the observations of \citet{ramsey2006adjacency}.

\subsection{Error correction}

One might wonder whether it is also possible to correct certain errors.
As noted before, e.g., the procedure in \cref{prop:markov_network_markovian}, but also  other algorithms like PC or SP can be sensitive to the test results in the sense that if a single one of the tests had a different result, the output of the algorithm would change.
As we will see, we can again use redundant CI-statements to correct errors.
And similarly to the case of error detection, conducting GA-redundant tests might be misleading.

\begin{example}[Graphoid prevents correction]
	\label{ex:graphoid_hurts_correction}
	Consider the graph in \cref{fig:graphoid_hurts_correction_example} and suppose we use the PC algorithm to learn it.
	First, the algorithm conducts all marginal CI-tests.
	Suppose they return the correct result except for\footnote{Note that such an error can never be ruled out with standard statistical testing frameworks, as the probability of a type I error cannot be zero for non-trivial procedures. 
    } $Y\not\ind Z$.
	In the next step, the algorithm would conduct the CI-tests with a non-empty conditioning set.
	So, if $Y\not\ind Z$ were the only error, one could hope that $Y\ind Z\mid X$ (implied by the ground truth graph) can still correct this mistake.
	But note that the marginal tests already imply $X\not\ind Y\mid Z$ and $Y\not\ind Z\mid X$.
	Intuitively, this is due to the \emph{in}dependence $X\ind Z$, which prevents conditioning on $X$ from changing the relationship between $Y$ and $Z$.
	This means, according to \cref{prop:continuity_mi}, we are likely to get the wrong test result for the conditional \emph{in}dependence between $Y$ and $Z$.
	
	One might wonder whether this is a shortcoming of the PC algorithm.
	But note how this would also affect our result if we simply select the graph with the fewest contradicting CI-statements to the empirical independence model (which we will discuss in \cref{sec:mmd}).
	If we still get $X\ind Z\mid Y$ right, there are four tests in favour of the actual ground truth model. 
    But there are five that would be explained, e.g., by the graph $X\to Y\to Z$ (see \cref{sec:small_claims}). 
	In other words, by adding GA-redundant tests into our consideration, we might wrongly conclude that we have more evidence for the wrong graph.
			\begin{figure}[tp]
		\centering
		\begin{tikzpicture}
			\node[obs] (X) {$X$} ;
			\node[obs, right=of X] (Y) {$Y$} edge[<-] (X) ;
			\node[obs, right =of Y] (Z) {$Z$} ;
			
		\end{tikzpicture}
		\caption{A false test $Y\not\ind Z$ may lead to the conclusion that the true graph is $X\to Y\to Z$. If this were the only error, the test $Y\ind Z\mid X$ would correct that. But $Y\not\ind Z\mid X$ follows via Graphoid axioms from the marginal tests.}
		\label{fig:graphoid_hurts_correction_example}
	\end{figure}
\end{example}
\paragraph{Error correction in spanning trees}

The following result shows that we can correct errors if we consider more tests than necessary.
To circumvent the issues raised in \cref{ex:graphoid_hurts_correction}, we only consider tests with a conditioning set of size one.
This way, we get a set of tests that are not restricted by Graphoid axioms (up to axiom 1) but suffice to identify a spanning tree.
\begin{lemma}
\label{lem:tree_tests_non_graphoid}
		Let $L = \{\CI(X, Y\,|\, Z) :\:X, Y, Z\in V,\, X\neq Y,\, X\neq Z,\, Y\neq Z\}$ such that $\CI(X, Y\mid Z) = \CI(Y, X\mid Z)$ for all distinct $X, Y, Z\in V$. 
        Then $L$ contains no contradictions with respect to the Graphoid axioms and no other CI-statement follows from $L$ via Graphoid axioms.
\end{lemma}
Then we can indeed simply pick the \enquote{message} whose encoding has the smallest distance to the received code word, i.e., the tree whose independence model differs the least from the observed one.

\begin{proposition}[error correction in trees]
	\label{prop:error_correction_trees}
    Let the set $S = \{(X, Y, Z) :\: X, Y, Z\in V,\, X\neq Y,\, X\neq Z, Y\neq Z\}$.
	Further, let $\cT_n$ be the set of spanning trees with $n\in \N_{> 3}$ nodes, $T^*\in \cT_n$ and $M$ be an independence model with $\MD_S(T^*, M) \le \lfloor (n-1)/2 \rfloor$, where  
    \begin{displaymath}
        \MD_S(T, M) = \sum_{s\in S} \mathbb{I}\left[(s\in M_{T}) \neq (s\in M)\right].
    \end{displaymath} 
    Then we can correct  at least $\lfloor (n-1)/2 \rfloor$ errors for any spanning tree $T^*$ by minimising the distance to $M$, i.e.
	\begin{displaymath}
		T^* = \argmin_{T\in \cT_n} \MD_S(T, M).
	\end{displaymath}
\end{proposition}

\paragraph{Error correction in DAGs}
\label{subsec:dags}
The strong graphical assumptions in the previous section enabled us to derive the guarantee in \cref{prop:error_correction_trees}.
It would be desirable to have a similar result for DAGs.
As \cref{ex:complete_dag_single_error} shows, this is not possible without further restrictions.
\begin{example}[almost complete DAG]
	\label{ex:complete_dag_single_error}
	Let $G$ be a complete DAG over $\bV$, i.e. for $n\in\N_{>1}$ the graph with nodes $\{X_1, \dots, X_n\}= \bV$ and edges $X_i\to X_j$ whenever $i < j$.
	Suppose our tests unfaithfully show $X_{n-1}\ind X_n\mid \bV\setminus\{X_{n-1}, X_n\}$.
	This observed independence model would be explained by the graph 
    \begin{displaymath}
    G-(X_{n-1}\to X_n). 
    \end{displaymath}
	Even though this is only a single error, there is another graph that perfectly explains the independence model, so 
    we would prefer this graph over $G$.
\end{example}
One might consider a subset of tests that cannot contain implications about each other, like we already did in \cref{prop:error_correction_trees}.
As we have seen in \cref{ex:cpc_not_graphoid_redundant}, a potential candidate could be the tests that the CPC algorithm by \citet{ramsey2006adjacency} utilises to orient colliders.
Indeed, \citet{colombo2014order} have proposed to do a majority vote over these tests (although they did not investigate the GA-redundancy of these tests).
Clearly, this method is capable of correcting errors.
It would be interesting to characterise further such \enquote{local} criteria, where a subgraph of the learned DAG can be corrected.
But note how \citet{colombo2014order} rely on the assumption that the learned skeleton is correct.
Without such an assumption, it is not obvious how a local error correction could be established.
\todo{Better than last version or too defensive?}


A different approach would be to study under which conditions the optimization over \emph{all} tests works.
Indeed, recently there has been a lot of interest in methods like the \emph{sparsest permutation} (SP) algorithm \citep{raskutti2018learning,lam2022greedy,andrews2023fast}, which outputs the sparsest graph among the graphs that can be constructed like in \cref{prop:protocol_graph_markovian} for all permutations of the nodes.
\citet{raskutti2018learning} show that the required assumption is strictly weaker than faithfulness.
Moreover, they postulate a \enquote{statistical/computational trade-off}, i.e., that additional computations can help to reduce statistical uncertainty.
Although they do not formally analyse this statement, the idea is in line with our work.
To emphasise this, we will briefly study SP from the perspective of redundancy.
The following lemma shows that SP relies on two key aspects.
\begin{lemma}[characterisation of SP output] \label{lem:sp}
The SP algorithm outputs a DAG $G^*$ iff
\begin{enumerate}[(a)] \todo{replace tests with CI-statements}
\item   there is a topological ordering $\pi^*$ with respect to $G^*$ such that all tests in $L_{\pi^*}$ are as in $M_{G^*}$.
\item  for all other permutations $\pi'$ there are no less than $|G^*|$ dependences in $L_{\pi'}$.
\end{enumerate} 
\end{lemma}
Evidently, if (a) fails, the algorithm cannot recover from this error.
But in principle, SP can correct arbitrary errors in $L_{\pi'}$, as long as property (b) holds.
If we assume that the underlying independence model is Graphoid, the errors that can occur are already restricted.
\Cref{prop:protocol_graph_markovian} implies that given (a) all \emph{in}dependences are GA-redundant and thus we are unlikely to find a contradiction here.
Further, \cref{prop:coupling} from \cref{sec:graphical_crit_dependences} also characterises dependences that follow from (a).
This means the errors that the SP algorithm can correct are, for example, of the kind of \cref{prop:sufficient_cond_non_graphoid}.
In principle, it would be desirable to have an algorithm that can also handle violations of (a), while still requiring weaker assumptions than faithfulness.
In analogy to \cref{prop:error_correction_trees}, we construct an algorithm that fulfils this requirement in \cref{sec:mmd}.


\begin{figure*}[t]
    \centering
    \subfloat[The $p$-values of different CI-tests that are implied by known tests to be either dependent or independent. \dominik{could be interesting to see the plot also for the set of wrong tests only}]{
    \includegraphics[width=.3\linewidth]{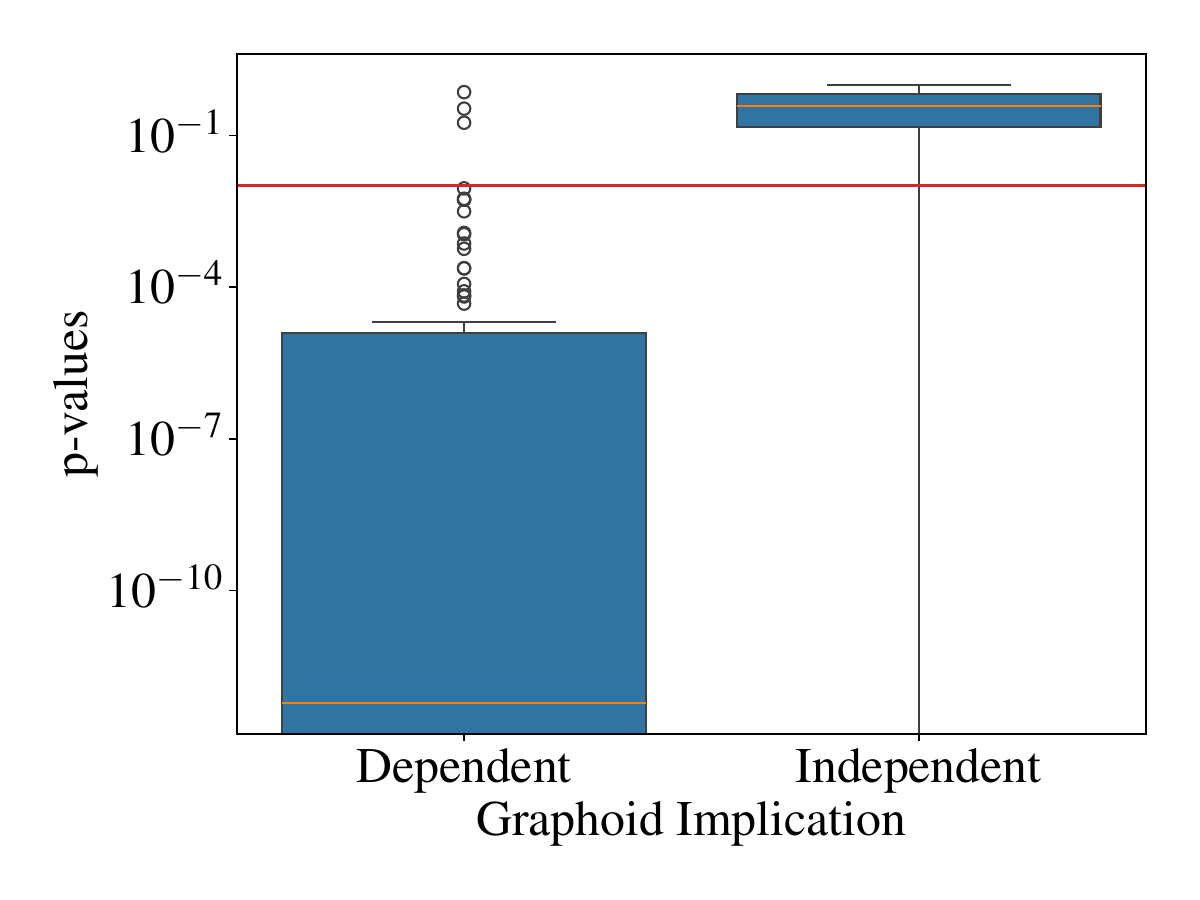}
	\label{fig:p_values_redundant}
    }\hspace{.5em}
    \subfloat[Incorrect predictions of PGM-redundant CI-statements for different models and data.]{
    \includegraphics[width=.3\linewidth]{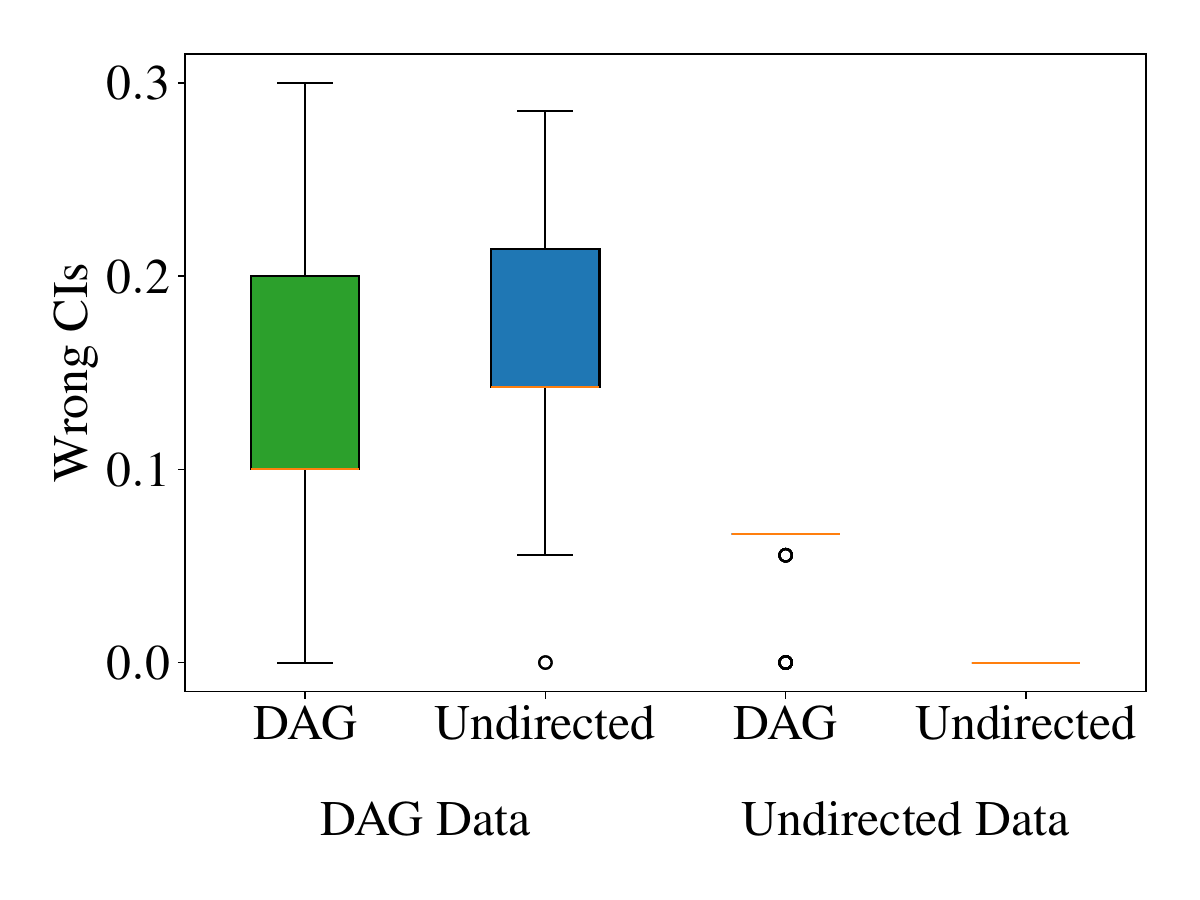}
	\label{fig:graphically_redundant_two_datasets}
    
    }    \hspace{.5em}
    \subfloat[Incorrect predictions of PGM- versus GA-redundant CI-statements on real data.\todo{GA-redundant}]{
    \includegraphics[width=.3\linewidth]{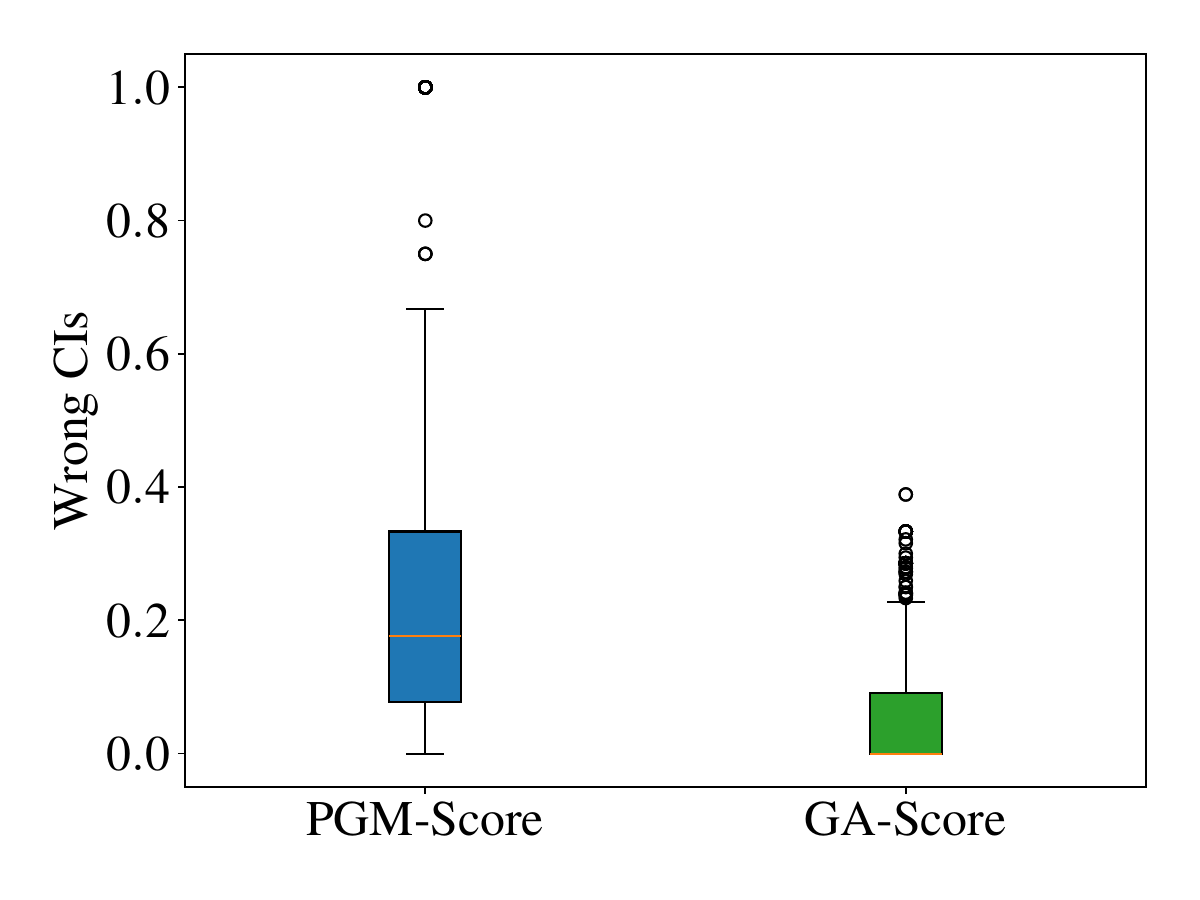}
    \label{fig:graphoid_vs_graphical_sachs}
    
    }
    \caption{Experimental results\vspace{-\baselineskip}}
    \label{fig:enter-label}
\end{figure*}

\section{EXPERIMENTS}
See \cref{sec:experiment_details} for more details on all experiments.
The source code can be found 
under \href{https://github.com/PhilippFaller/RedundancyInGraphDiscovery}{\url{https://github.com/PhilippFaller/RedundancyInGraphDiscovery}}. 
In the first experiment, we checked the hypothesis that empirical tests rarely contradict the Graphoid axioms.
To this end, we generated synthetic data from a multivariate Gaussian distribution with four variables 
and conducted several CI-tests.
Before each test, we check with the \texttt{Z3} solver \citep{z3} whether the result is already implied by the previous tests via Graphoid axioms.
If so, we track what result the axioms imply and the resulting $p$-value of a CI-test.
As we can see in \cref{fig:p_values_redundant}, the $p$-values mostly follow the predictions (where the red line indicates the significance level $0.01$).
This corroborates our hypothesis that GA-redundant tests provide little additional information and could thus give a false impression of evidence.

In the next experiments, 
we thus exclude GA-redundant tests from the evaluation of our model.
In the former, we synthetically generated two datasets with four binary variables.
One follows the DAG with edges $X\leftarrow W\to Y$ and $X\to Z\leftarrow Y$, while the other one follows the undirected model with the same skeleton.
We then learn a DAG and an undirected model with the procedures described in \cref{prop:protocol_graph_markovian,prop:markov_network_markovian} and identify GA-redundant CI-statements using \cref{prop:coupling,prop:markov_network_markovian,prop:protocol_graph_markovian}.
We then check whether they hold empirically in the data.
\Cref{fig:graphically_redundant_two_datasets} shows the fraction of these tests where the graphical implication and the empirical tests disagree.
We can see that the models  tend to make fewer wrong predictions when they match the respective data-generating process (green) than the other model (blue).

Finally, we learned a DAG on the protein signaling data from \citet{sachs2005causal}. 
We use the procedure from \cref{prop:protocol_graph_markovian} and identified PGM-redundant CI-statements via \cref{prop:sufficient_cond_non_graphoid}.
\Cref{fig:graphoid_vs_graphical_sachs} shows the fraction of tests where the graph contradicts the empirical result.
We further calculate this fraction with additional GA-redundant CI-statements that we detect via \cref{prop:protocol_graph_markovian,prop:coupling}.
In \cref{fig:graphoid_vs_graphical_sachs} we can see that the PGM-redundant tests indicate more errors than the GA-redundant tests.
This is in line with the fact that none of the recovered graphs were consistent with the ground truth provided by \citet{sachs2005causal}.
In \cref{sec:graphoid_vs_graphical_synthetic} we repeated the experiment on different synthetic datasets. 

\section{RELATED WORK}
We are the first to study how graphical and probabilistic constraints on CI-statements interplay in the \emph{detection} and \emph{correction} of errors in graph discovery. 
The fact that causal graphs can entail implications about parts of the distribution that were not seen before has been noted by \citet{tsamardinos2012towards,janzing2023reinterpreting} and in terms of structural causal models by \citep{gresele2022causal}.
Building on these insights, \citet{faller2024self,schkoda2025cross} have proposed to falsify causal models by exploiting these constraints.
Although this is the basic observation of GM-redundancy, we are the first to contrast this with P-redundancy.
\citet{mazaheri2025meta} show that there are different kinds of dependences between CI-statements, but do not discuss how this influences their confirmatory power with respect to graphical models.
On the other hand, \citet{faltenbacher2025internal} propose to use the CI-statements that PC would conduct anyway to detect contradictions without discussing the dependences between tests.
\citet{bromberg2009improving} have also noted that contradictions between CI-tests can be corrected to improve discovery results, and
\citet{kim2024causal} correct errors by using Graphoid axioms to conduct a set of tests that are 
statistically better conditioned.
In contrast to us, they both focus on violations of the Graphoid axioms.
We have argued before that such violations are comparatively rare, which does not mean that they are not informative if they occur. 
But additionally, we pointed out that the absence of these violations does not constitute evidence for a model.
\citet{hyttinen2014constraint,russo2024argumentative} propose to use symbolic reasoning to resolve conflicts in the provided CI-statements.
But neither makes the distinction between GM- and P-redundant tests.\todo{Need to understand this one better.}
In \cref{sec:additional_related_work}, we discuss more work on the general robustness of  PC.

Score-based graph learning methods have also been shown to recover the independence model under certain conditions \citep{aragam2015learning}. Arguably, these methods have the advantage that they \enquote{weigh} conflicting CI-statements by their influence on the score (which is often related to the likelihood).
On the other hand, it is not clear how, for such algorithms, one could determine which information has not been used in the optimization, or in other words, what test could be used to independently evaluate a learned graphical model. 
Studying this constitutes a research project in its own right. 

\citet{zhang2016three} argue that the faithfulness assumption serves three functions in graph discovery.
Our work can be interpreted as adding a fourth \enquote{face} to their three faces of faithfulness by showing that, in many cases, 
we get error detection and correction properties only through graphical assumptions (such as faithfulness) that are stronger than the laws of probability.

\section{DISCUSSION AND LIMITATIONS}
We have defined different notions of redundancy to distinguish between CI-statements that are already implied by the laws of probability and the ones that follow only from graphical assumptions.
As the former, GA-redundant ones, can wrongly give the impression of additional evidence, we have characterised conditions where we have the latter, PGM-redundancy, and conditions under which we only have GA-redundancy. We are the first to propose to use PGM-redundant tests similarly to held-out data in cross-validation, or redundant bits in coding theory. 

Our work shows that numerous correctly predicted CI-statements only provide evidence if they represent \enquote{independent degrees of freedom} of the underlying distribution.
This means that the mere number of correct CI-statements provides a superficial image and cannot corroborate a model.
This exacerbates the more tests are used (for robustness or evaluation), rendering it all the more important to avoid P-redundancy. 
On the other hand, our theory is fundamentally limited by the fact that graphical models can never be \emph{verified} without ground truth data or at least assumptions about the error model of the CI-tests used. 
Specifically, this means that we also cannot know for additional tests whether they are correct or how they might depend on each other in different ways than through Graphoid axioms. 
Nonetheless, we think that our insights can be a first step towards methods that not only detect or correct errors but also allow us to define confidence regions outside of which a model should be rejected, as the model differs too much from the observed evidence.
Another limitation is that additional CI-tests add to the long runtime of state-of-the-art methods, which cannot be avoided unless $\mathcal{P}=\mathcal{NP}$ \citep{chickering1996learning}. 
Regardless, we think that more robust discovery, and discovery with some form of quality estimate, is more useful for downstream tasks than a highly scalable method without that, as results are often  sensitive even to small errors.

\subsubsection*{Acknowledgements}
We thank Elke Kirschbaum and William Roy Orchard for helpful discussions and proofreading.
Philipp M. Faller was supported by a doctoral scholarship of the Studienstiftung des deutschen Volkes (German Academic Scholarship Foundation).
This work does not relate to Dominik Janzing's position at Amazon.

\bibliography{paper}

@article{shah2020hardness,
	title={The hardness of conditional independence testing and the generalised covariance measure},
	author={Shah, Rajen D and Peters, Jonas},
    volume = {48},
    journal = {The Annals of Statistics},
    publisher = {Institute of Mathematical Statistics},
    pages = {1514 -- 1538},
    year={2020}
}

@inproceedings{eulig2025toward,
  title={Toward falsifying causal graphs using a permutation-based test},
  author={Eulig, Elias and Mastakouri, Atalanti A and Bl{\"o}baum, Patrick and Hardt, Michaela and Janzing, Dominik},
  booktitle={Proceedings of the AAAI Conference on Artificial Intelligence},
  volume={39},
  pages={26778--26786},
  year={2025}
}

@book{lauritzen1996graphical,
	title={Graphical models},
	author={Lauritzen, Steffen L},
	volume={17},
	year={1996},
	publisher={Clarendon Press}
}

@article{janzing2023reinterpreting,
	title={Reinterpreting causal discovery as the task of predicting unobserved joint statistics},
	author={Janzing, Dominik and Faller, Philipp M and Vankadara, Leena Chennuru},
	journal={arXiv preprint arXiv:2305.06894},
	year={2023}
}

@book{pearl2009causality,
	title={Causality},
	author={Pearl, Judea},
	year={2009},
	publisher={Cambridge university press}
}

@book{spirtes2000causation,
	title={Causation, prediction, and search},
	author={Spirtes, Peter and Glymour, Clark N and Scheines, Richard and Heckerman, David},
	year={2000},
	publisher={MIT press}
}

@article{tsamardinos2006max,
	title={The max-min hill-climbing Bayesian network structure learning algorithm},
	author={Tsamardinos, Ioannis and Brown, Laura E and Aliferis, Constantin F},
	journal={Machine learning},
	volume={65},
	number={1},
	pages={31--78},
	year={2006},
	publisher={Springer}
}

@article{kalisch2007estimating,
	title={Estimating high-dimensional directed acyclic graphs with the PC-algorithm.},
	author={Kalisch, Markus and B{\"u}hlman, Peter},
	journal={Journal of Machine Learning Research},
	volume={8},
	number={3},
	year={2007}
}

@article{uhler2013geometry,
	title={Geometry of the faithfulness assumption in causal inference},
	author={Uhler, Caroline and Raskutti, Garvesh and B{\"u}hlmann, Peter and Yu, Bin},
	journal={The Annals of Statistics},
	pages={436--463},
	year={2013},
	publisher={JSTOR}
}

@inproceedings{meek1995causal,
	title={Causal inference and causal explanation with background knowledge},
	author={Meek, Christopher},
	booktitle={Proceedings of the Eleventh conference on Uncertainty in artificial intelligence},
	pages={403--410},
	year={1995}
}

@article{strobl2019estimating,
  title={Estimating and controlling the false discovery rate of the pc algorithm using edge-specific p-values},
  author={Strobl, Eric V and Spirtes, Peter L and Visweswaran, Shyam},
  journal={ACM Transactions on Intelligent Systems and Technology (TIST)},
  volume={10},
  number={5},
  pages={1--37},
  year={2019},
  publisher={ACM New York, NY, USA}
}

@article{kalisch2008robustification,
	title={Robustification of the PC-algorithm for directed acyclic graphs},
	author={Kalisch, Markus and B{\"u}hlmann, Peter},
	journal={Journal of Computational and Graphical Statistics},
	volume={17},
	number={4},
	pages={773--789},
	year={2008},
	publisher={Taylor \& Francis}
}

@article{li2009controlling,
	title={Controlling the False Discovery Rate of the Association/Causality Structure Learned with the PC Algorithm.},
	author={Li, Junning and Wang, Z Jane},
	journal={Journal of Machine Learning Research},
	volume={10},
	number={2},
	year={2009}
}

@article{harris2013pc,
	title={PC algorithm for nonparanormal graphical models.},
	author={Harris, Naftali and Drton, Mathias},
	journal={Journal of Machine Learning Research},
	volume={14},
	number={11},
	year={2013}
}

@inproceedings{korhonen2024structural,
	title={Structural perspective on constraint-based learning of Markov networks},
	author={Korhonen, Tuukka and Fomin, Fedor and Parviainen, Pekka},
	booktitle={International Conference on Artificial Intelligence and Statistics},
	pages={1855--1863},
	year={2024},
	organization={PMLR}
}

@inproceedings{zhang2024membership,
	title={Membership testing in markov equivalence classes via independence queries},
	author={Zhang, Jiaqi and Shiragur, Kirankumar and Uhler, Caroline},
	booktitle={International Conference on Artificial Intelligence and Statistics},
	pages={3925--3933},
	year={2024},
	organization={PMLR}
}

@inproceedings{shiragur2024causal,
  title={Causal Discovery with Fewer Conditional Independence Tests},
  author={Shiragur, Kirankumar and Zhang, Jiaqi and Uhler, Caroline},
  booktitle={International Conference on Machine Learning},
  pages={45060--45078},
  year={2024},
  organization={PMLR}
}

@article{textor2016robust,
	title={Robust causal inference using directed acyclic graphs: the R package ‘dagitty’},
	author={Textor, Johannes and Van der Zander, Benito and Gilthorpe, Mark S and Li{\'s}kiewicz, Maciej and Ellison, George TH},
	journal={International journal of epidemiology},
	volume={45},
	number={6},
	pages={1887--1894},
	year={2016},
	publisher={Oxford University Press}
}

@book{mackay2003information,
	title={Information theory, inference and learning algorithms},
	author={MacKay, David JC},
	year={2003},
	publisher={Cambridge university press}
}

@incollection{pearl2022graphoids,
	title={Graphoids: Graph-Based Logic for Reasoning about Relevance Relations or When would x tell you more about y if you already know z?},
	author={Pearl, Judea and Paz, Azaria},
	booktitle={Probabilistic and Causal Inference: The Works of Judea Pearl},
	pages={189--200},
    publisher = {Association for Computing Machinery},
	year={2022}
}

@article{raskutti2018learning,
	title={Learning directed acyclic graph models based on sparsest permutations},
	author={Raskutti, Garvesh and Uhler, Caroline},
	journal={Stat},
	volume={7},
	number={1},
	pages={e183},
	year={2018},
	publisher={Wiley Online Library}
}

@phdthesis{bouckaert1995bayesian,
	title={Bayesian belief networks: from construction to inference},
	author={Bouckaert, Remco Ronaldus},
    school={Utrecht University},
	year={1995}
}

@article{colombo2014order,
	title={Order-independent constraint-based causal structure learning.},
	author={Colombo, Diego and Maathuis, Marloes H and others},
	journal={J. Mach. Learn. Res.},
	volume={15},
	number={1},
	pages={3741--3782},
	year={2014}
}

@inproceedings{ramsey2006adjacency,
	title={Adjacency-faithfulness and conservative causal inference},
	author={Ramsey, Joseph and Spirtes, Peter and Zhang, Jiji},
	booktitle={Proceedings of the Twenty-Second Conference on Uncertainty in Artificial Intelligence},
	pages={401--408},
	year={2006}
}

@article{geiger1993logical,
	title={Logical and algorithmic properties of conditional independence and graphical models},
	author={Geiger, Dan and Pearl, Judea},
	journal={The annals of statistics},
	volume={21},
	number={4},
	pages={2001--2021},
	year={1993},
	publisher={Institute of Mathematical Statistics}
}

@incollection{verma1990causal,
	title={Causal networks: Semantics and expressiveness},
	author={Verma, Thomas and Pearl, Judea},
	booktitle={Machine intelligence and pattern recognition},
	volume={9},
	pages={69--76},
	year={1990},
	publisher={Elsevier}
}

@inproceedings{studeny1992conditional,
	title={Conditional independence relations have no finite complete characterization},
	author={Studen{\`y}, Milan},
	booktitle={Information Theory, Statistical Decision Functions and Random Processes. Transactions of the 11th Prague Conference vol. B},
	pages={377--396},
	year={1992}
}

@inproceedings{z3,
	title={Z3: An efficient SMT solver},
	author={De Moura, Leonardo and Bj{\o}rner, Nikolaj},
	booktitle={International conference on Tools and Algorithms for the Construction and Analysis of Systems},
	pages={337--340},
	year={2008},
	organization={Springer}
}

@article{andersen2013expect,
	title={When to expect violations of causal faithfulness and why it matters},
	author={Andersen, Holly},
	journal={Philosophy of Science},
	volume={80},
	number={5},
	pages={672--683},
	year={2013},
	publisher={Cambridge University Press}
}

@article{zhang2016three,
	title={The three faces of faithfulness},
	author={Zhang, Jiji and Spirtes, Peter},
	journal={Synthese},
	volume={193},
	pages={1011--1027},
	year={2016},
	publisher={Springer}
}

@article{richardson2003markov,
  title={Markov properties for acyclic directed mixed graphs},
  author={Richardson, Thomas},
  journal={Scandinavian Journal of Statistics},
  volume={30},
  number={1},
  pages={145--157},
  year={2003},
  publisher={Wiley Online Library}
}

@inproceedings{wahl2025separation,
  title={Separation-Based Distance Measures for Causal Graphs},
  author={Wahl, Jonas and Runge, Jakob},
  booktitle={International Conference on Artificial Intelligence and Statistics},
  pages={3412--3420},
  year={2025},
  organization={PMLR}
}

@inproceedings{schkoda2025cross,
  title={Cross-validating causal discovery via Leave-One-Variable-Out},
  author={Schkoda, Daniela and Faller, Philipp Michael and Janzing, Dominik and Bl{\"o}baum, Patrick},
  booktitle={Causal Learning and Reasoning},
  pages={659--692},
  year={2025},
  organization={PMLR}
}

@inproceedings{faller2024self,
  title={Self-compatibility: Evaluating causal discovery without ground truth},
  author={Faller, Philipp M and Vankadara, Leena C and Mastakouri, Atalanti A and Locatello, Francesco and Janzing, Dominik},
  booktitle={International Conference on Artificial Intelligence and Statistics},
  pages={4132--4140},
  year={2024},
  organization={PMLR}
}

@article{tsamardinos2012towards,
  title={Towards integrative causal analysis of heterogeneous data sets and studies},
  author={Tsamardinos, Ioannis and Triantafillou, Sofia and Lagani, Vincenzo},
  journal={The Journal of Machine Learning Research},
  volume={13},
  number={1},
  pages={1097--1157},
  year={2012},
  publisher={JMLR. org}
}

@inproceedings{gresele2022causal,
  title={Causal inference through the structural causal marginal problem},
  author={Gresele, Luigi and Von K{\"u}gelgen, Julius and K{\"u}bler, Jonas and Kirschbaum, Elke and Sch{\"o}lkopf, Bernhard and Janzing, Dominik},
  booktitle={International Conference on Machine Learning},
  pages={7793--7824},
  year={2022},
  organization={PMLR}
}

@inproceedings{bhattacharyya2021near,
  title={Near-optimal learning of tree-structured distributions by Chow-Liu},
  author={Bhattacharyya, Arnab and Gayen, Sutanu and Price, Eric and Vinodchandran, NV},
  booktitle={Proceedings of the 53rd annual acm SIGACT symposium on theory of computing},
  pages={147--160},
  year={2021}
}

@InProceedings{kim2024causal,
  title = 	 {Causal Discovery with Deductive Reasoning: One Less Problem},
  author =       {Kim, Jonghwan and Hwang, Inwoo and Lee, Sanghack},
  booktitle = 	 {Proceedings of the Fortieth Conference on Uncertainty in Artificial Intelligence},
  pages = 	 {1999--2017},
  year = 	 {2024},
  series = 	 {Proceedings of Machine Learning Research},
  publisher =    {PMLR}
  
}

@article{bromberg2009improving,
  title={Improving the Reliability of Causal Discovery from Small Data Sets Using Argumentation.},
  author={Bromberg, Facundo and Margaritis, Dimitris},
  journal={Journal of Machine Learning Research},
  volume={10},
  number={2},
  year={2009}
}

@inproceedings{russo2024argumentative,
  title={Argumentative causal discovery},
  author={Russo, Fabrizio and Rapberger, Anna and Toni, Francesca},
  booktitle={Proceedings of the 21st International Conference on Principles of Knowledge Representation and Reasoning},
  pages={938--949},
  year={2024}
}

@article{zheng2024causal,
  title={Causal-learn: Causal discovery in python},
  author={Zheng, Yujia and Huang, Biwei and Chen, Wei and Ramsey, Joseph and Gong, Mingming and Cai, Ruichu and Shimizu, Shohei and Spirtes, Peter and Zhang, Kun},
  journal={Journal of Machine Learning Research},
  volume={25},
  number={60},
  pages={1--8},
  year={2024}
}

@inproceedings{lam2022greedy,
  title={Greedy relaxations of the sparsest permutation algorithm},
  author={Lam, Wai-Yin and Andrews, Bryan and Ramsey, Joseph},
  booktitle={Uncertainty in Artificial Intelligence},
  pages={1052--1062},
  year={2022},
  organization={PMLR}
}

@article{andrews2023fast,
  title={Fast scalable and accurate discovery of dags using the best order score search and grow shrink trees},
  author={Andrews, Bryan and Ramsey, Joseph and Sanchez Romero, Ruben and Camchong, Jazmin and Kummerfeld, Erich},
  journal={Advances in Neural Information Processing Systems},
  volume={36},
  pages={63945--63956},
  year={2023}
}

@article{ankan2024pgmpy,
  title={pgmpy: A Python toolkit for Bayesian networks},
  author={Ankan, Ankur and Textor, Johannes},
  journal={Journal of Machine Learning Research},
  volume={25},
  number={265},
  pages={1--8},
  year={2024}
}

@article{mazaheri2025meta,
  title={Meta-Dependence in Conditional Independence Testing},
  author={Mazaheri, Bijan and Zhang, Jiaqi and Uhler, Caroline},
  journal={arXiv preprint arXiv:2504.12594},
  year={2025}
}

@article{faltenbacher2025internal,
  title={Internal Incoherency Scores for Constraint-based Causal Discovery Algorithms},
  author={Faltenbacher, Sofia and Wahl, Jonas and Herman, Rebecca and Runge, Jakob},
  journal={arXiv preprint arXiv:2502.14719},
  year={2025}
}

@book{bishop2006pattern,
  title={Pattern recognition and machine learning},
  author={Bishop, Christopher M and Nasrabadi, Nasser M},
  volume={4},
  year={2006},
  publisher={Springer}
}

@article{sachs2005causal,
	title={Causal protein-signaling networks derived from multiparameter single-cell data},
	author={Sachs, Karen and Perez, Omar and Pe'er, Dana and Lauffenburger, Douglas A and Nolan, Garry P},
	journal={Science},
	volume={308},
	number={5721},
	pages={523--529},
	year={2005},
	publisher={American Association for the Advancement of Science}
}

@article{aragam2015learning,
	title={Learning directed acyclic graphs with penalized neighbourhood regression},
	author={Aragam, Bryon and Amini, Arash A and Zhou, Qing},
	journal={arXiv preprint arXiv:1511.08963},
	year={2015}
}

@inproceedings{hyttinen2014constraint,
	title={Constraint-based Causal Discovery: Conflict Resolution with Answer Set Programming.},
	author={Hyttinen, Antti and Eberhardt, Frederick and J{\"a}rvisalo, Matti},
	booktitle={UAI},
	pages={340--349},
	year={2014}
}

@article{ramsey2016improving,
	title={Improving accuracy and scalability of the pc algorithm by maximizing p-value},
	author={Ramsey, Joseph},
	journal={arXiv preprint arXiv:1610.00378},
	year={2016}
}

@article{li2019constraint,
	title={Constraint-based causal structure learning with consistent separating sets},
	author={Li, Honghao and Cabeli, Vincent and Sella, Nadir and Isambert, Herv{\'e}},
	journal={Advances in neural information processing systems},
	volume={32},
	year={2019}
}

@article{chickering1996learning,
  title={Learning Bayesian networks is NP-complete},
  author={Chickering, David Maxwell},
  journal={Learning from data: Artificial intelligence and statistics V},
  pages={121--130},
  year={1996},
  publisher={Springer}
}

@inproceedings{wienobst2020recovering,
  title={Recovering causal structures from low-order conditional independencies},
  author={Wien{\"o}bst, Marcel and Liskiewicz, Maciej},
  booktitle={Proceedings of the AAAI Conference on Artificial Intelligence},
  volume={34},
  pages={10302--10309},
  year={2020}
}

@article{kocaoglu2023characterization,
  title={Characterization and learning of causal graphs with small conditioning sets},
  author={Kocaoglu, Murat},
  journal={Advances in Neural Information Processing Systems},
  volume={36},
  pages={74140--74179},
  year={2023}
}

@article{rohekar2021iterative,
  title={Iterative causal discovery in the possible presence of latent confounders and selection bias},
  author={Rohekar, Raanan Y and Nisimov, Shami and Gurwicz, Yaniv and Novik, Gal},
  journal={Advances in Neural Information Processing Systems},
  volume={34},
  pages={2454--2465},
  year={2021}
}

@inproceedings{niepert2009logical,
  title={Logical inference algorithms and matrix representations for probabilistic conditional independence},
  author={Niepert, Mathias},
  booktitle={Proceedings of the Twenty-Fifth Conference on Uncertainty in Artificial Intelligence},
  pages={428--435},
  year={2009}
}

\section*{Checklist}

\begin{enumerate}

  \item For all models and algorithms presented, check if you include:
  \begin{enumerate}
    \item A clear description of the mathematical setting, assumptions, algorithm, and/or model. Yes.
    \item An analysis of the properties and complexity (time, space, sample size) of any algorithm. Yes.
    \item (Optional) Anonymized source code, with specification of all dependencies, including external libraries. Yes.
  \end{enumerate}

  \item For any theoretical claim, check if you include:
  \begin{enumerate}
    \item Statements of the full set of assumptions of all theoretical results. Yes.
    \item Complete proofs of all theoretical results. Yes.
    \item Clear explanations of any assumptions. Yes.
  \end{enumerate}

  \item For all figures and tables that present empirical results, check if you include:
  \begin{enumerate}
    \item The code, data, and instructions needed to reproduce the main experimental results (either in the supplemental material or as a URL). Yes.
    \item All the training details (e.g., data splits, hyperparameters, how they were chosen). Yes.
    \item A clear definition of the specific measure or statistics and error bars (e.g., with respect to the random seed after running experiments multiple times). Yes.
    \item A description of the computing infrastructure used. (e.g., type of GPUs, internal cluster, or cloud provider). Yes.
  \end{enumerate}

  \item If you are using existing assets (e.g., code, data, models) or curating/releasing new assets, check if you include:
  \begin{enumerate}
    \item Citations of the creator. If your work uses existing assets. Yes.
    \item The license information of the assets, if applicable. Yes.
    \item New assets either in the supplemental material or as a URL, if applicable. Yes.
    \item Information about consent from data providers/curators. Not Applicable.
    \item Discussion of sensible content if applicable, e.g., personally identifiable information or offensive content. Not Applicable.
  \end{enumerate}

  \item If you used crowdsourcing or conducted research with human subjects, check if you include:
  \begin{enumerate}
    \item The full text of instructions given to participants and screenshots. Not Applicable.
    \item Descriptions of potential participant risks, with links to Institutional Review Board (IRB) approvals if applicable. Not Applicable.
    \item The estimated hourly wage paid to participants and the total amount spent on participant compensation. Not Applicable.
  \end{enumerate}

\end{enumerate}

\clearpage
\appendix
\thispagestyle{empty}

\onecolumn
\aistatstitle{On Different Notions of Redundancy in Conditional-Independence-Based Discovery of Graphical Models\\
Supplementary Materials}

\section{DETAILED DEFINITIONS}
\label{sec:further_definitions}
We denote a random variable with upper case letter $X$. 
A set of random variables is denoted with boldface letters $\bX$.
For single variables and sets of variables, we write the attained values with a lower case letter $x$ and the domain with calligraphic letter $\cX$.

A \emph{graph} $G$ is a tuple $(\bV, \bE)$, where $\bV$ is a finite set of nodes and $\bE \subseteq \bV\times \bV$ is the set of edges.
We say $G$ is \emph{undirected} if $\bE$ is a symmetric relation.
To emphasise the direction of an edge we also write $X\to Y$ for $(X, Y)\in \bE$, $X\leftarrow Y$ when $(Y, X)\in \bE$ and $X-Y$ when the graph is undirected and $(X, Y)\in \bE$.
We slightly abuse notation and write $(X\to Y)\in G$ if $(X, Y)\in \bE$ and analogously for $X\leftarrow Y$ and $X-Y$.
We denote the graph $G - (X\to Y) := (\bV, \bE\setminus\{(X, Y)\})$.
We say two nodes $X, Y\in \bV$ are \emph{adjacent} if we have $(X, Y)\in \bE$ or $(Y, X)\in \bE$ and we denote with $\Adj(X)$ the set of nodes that are adjacent to $X$.
$\PA(Y)$ denotes the set of nodes $X\in \bV$ such that $(X, Y)\in \bE$. 
A \emph{path} $p=(X_1, \dots, X_n)$ is a sequence of $n\in \N_{>1}$ nodes such that $X_i\in \bV$ for $i=1, \dots, n$ and $X_i\in \Adj(X_{i+1})$ for $i=1, \dots, n-1$.
Further, $p$ is called a \emph{cycle} if $X_1=X_n$.
If a graph contains no cycles, it is called a \emph{directed acyclic graph} (DAG).
If in an undirected graph $G$ any distinct nodes $X, Y\in \bV$ are connected by at most one path, $G$ is called a \emph{tree}.
If it is exactly one path, $G$ is called a \emph{spanning tree}.
For a DAG $G=(\bV, \bE)$ we define the \emph{skeleton} of $G$ as the undirected graph $H=(\bV, \bE')$, with $(X, Y)\in \bE'$ if $(X, Y)\in \bE$ or $(Y, X)\in \bE$, for $X, Y\in \bV$.
For all graphs $G=(\bV, \bE)$ we denote $|G| := |\bE|$.
A DAG $G = (\bV, \bE)$ is called a \emph{complete DAG} if we have $|\bE| = |\bV|(|\bV|-1) / 2$.

Let $\bV$ be a set of random variables.
We call a set $M \subseteq \mathcal{P}(\bV)^3$ an \emph{independence model}\footnote{Note that \citet{bouckaert1995bayesian} formally distinguishes between an \emph{independence} and \emph{dependence} model. For our discussion it should suffice to  consider the latter only implicitly.} if all $(\bX, \bY, \bZ) \in M$ are disjoint and $\bX$ and $\bY$ are not empty.
Then we say $\bX$ is \emph{independent} from $\bY$ given $\bZ$ and write $\bX\ind_{\mspace{-8mu}M}\bY\mid \bZ$, where we often omit the subscript.
If $(\bX, \bY, \bZ)$ is not in $M$ we say $\bX$ is \emph{dependent} on $\bY$ given $\bZ$ and write $\bX\not\ind \bY\mid \bZ$.
A \emph{CI-statement} is a quadruple of $\bX, \bY, \bZ$ and a boolean value, indicating whether the independence holds.
For a set of CI-statements $L$ we slightly abuse notation and write $L\subseteq M$ if for $(\bX, \bY, \bZ, b)\in L$ we have $b=((\bX,\bY,\bZ)\in M$).
We also sometimes write a CI-statement as function $\CI: \bX, \bY, \bZ\mapsto (\bX, \bY, \bZ, b)$. 
As we noted before, with \emph{in}dependences we refer to statements of the form $\bX\ind \bY\mid \bZ$ and with dependences to $\bX\not\ind \bY\mid \bZ$, while with CI-statement we refer to both of them.

A probability distribution over $\bV$ entails an independence model by the standard definition of probabilistic conditional independence, e.g. $\bX\ind \bY\mid \bZ$ with respect to the discrete distribution $P$ iff 
\begin{displaymath}
    P(\bX = \bx, \bY=\by\mid \bZ=\bz) = P(\bX=\bx\mid \bZ=\bz)\cdot P(\bY=\by\mid \bZ=\bz),
\end{displaymath}
for all $\bx\in\cX, \by\in \cY, \bz\in\mathcal{Z}$.
We often refer to the distribution and its independence model interchangeably.

We can also use graphs to represent independence models.
First, we define a graphical notion and then see how that corresponds to conditional independence.
Let $G$ be an undirected graph with nodes $\bV$ and $\bX, \bY, \bZ\subseteq \bV$ be disjoint with $\bX\neq \emptyset \neq \bY$.
We say $\bX$ and $\bY$ are \emph{separated} given $\bZ$ if every path from $X\in \bX$ to $Y\in \bY$ contains a node in $\bZ$.
If a path contains no node in $\bZ$, we say it is \emph{active}.
For a DAG $G=(\bV, \bE)$, and a path $p=(X_1, \dots, X_n)$ in $G$ we say $X_i$ is a \emph{collider} on $p$ if $(X_{i-1}\to X_i), (X_i \leftarrow X_{i+1}) \in \bE$, for $n\in \N_{>2}, X_1, \dots, X_n \in \bV, i\in \{2, \dots, n-1\}$.
We say a path $p$ is \emph{active} given $\bZ$ if for all colliders $C$ on $p$, a descendant of $C$ or $C$ itself is in $\bZ$.
The sets $\bX$ and $\bY$ are \emph{$d$-separated} if there are no $X\in \bX$ and $Y\in \bY$ such that there is an active path given $\bZ$ between $X$ and $Y$.
Now we denote a model induced by a graph $G$ as $M_G$.
For an undirected graph $G$ we define $(\bX, \bY, \bZ) \in M_G$ iff $\bX$ is separated from $\bY$ given $\bZ$.
For DAGs $(\bX, \bY, \bZ) \in M_G$ iff $\bX$ is $d$-separated from $\bY$ given $\bZ$.
By \emph{graphical model} we refer to \emph{either} an undirected graph or a DAG (and its respective independence model).
As we have noted before, we restrict our attention to undirected graphs and DAGs.
But our insights can be applied to any model that is equipped with a notion of independence between nodes, such as chain graphs \citep{lauritzen1996graphical}, completed partial DAGs, maximal ancestral graphs, partial ancestral graphs \citep{spirtes2000causation}, or acyclic directed mixed graphs \citep{richardson2003markov}.
If for two graphs $G, G'$ we have $M_G = M_{G'}$, we say $G$ and $G'$ are \emph{Markov equivalent} and for any graph $G$ we call the set of DAGs that are Markov equivalent to $G$ the \emph{Markov equivalence class} of $G$.

A graph is \emph{Markovian} to an independence model $M$ if $(\bX, \bY, \bZ) \in M_G \implies(\bX, \bY, \bZ) \in M$ and it is \emph{faithful} if $(\bX, \bY, \bZ) \in M \implies (\bX, \bY, \bZ) \in M_G$.
A graph is called a \emph{minimal I-map} if it is Markovian but removing any edge would make it non-Markovian.
Again, we slightly abuse notation and say $G$ is Markovian to a set of CI-statements $L$ when all \emph{in}dependences in $G$ are contained in $L$.
For independence models $M, M'$ over $\bV$ we define the \emph{Markov distance} with respect to $S\subseteq \CI(\bV)$ via
\begin{displaymath}
   \MD_S(M, M') = \sum_{s\in S} \mathbb{I}\left[(s\in M) \neq (s\in M')\right].
\end{displaymath}
We omit the subscript if $S=\CI(\bV)$.
In this case, \citet{wahl2025separation} call this \emph{s/c-metric} and show that it is a proper metric for the space of Markov equivalence classes.
We extend the definition to a graph $G$ by considering the induced independence model $M_G$, i.e. we define $\MD_S(G, M) = \MD_S(M_G, M)$ and $\MD_S(M, G) = \MD_S(M, M_G)$.


\citet{bouckaert1995bayesian} uses the following definition to graphically characterise CI-statements that follow via the Graphoid axioms (\cref{def:graphoid}), as used in \cref{prop:coupling}.
\begin{definition}[coupling]
\label{def:coupling}
    Let $G$ be an undirected graph over $\bV$ and $\bX, \bY, \bZ \subseteq \bV$ be disjoint with $\bX\neq\emptyset\neq\bY$.
    Then $\bX$ and $\bY$ are \emph{coupled} given $\bZ$ if there are $X\in \bX, Y\in \bY$ or $Y\in \bX, X\in\bY$ such that
    \begin{displaymath}
        (X-Y)\in G \quad \text{ and } \quad
        \Adj(X)\subset \bX\cup\bY\cup\bZ.
    \end{displaymath}
    
	Now let $G$ be a DAG. Then
    $\bX$ and $\bY$ are \emph{coupled} given $\bZ$ if there are $X\in \bX, Y\in \bY$ or $Y\in \bX, X\in\bY$ such that
    \begin{displaymath}
        (X\to Y)\in G, \qquad 
        \PA(Y)\subset \bX\cup\bY\cup\bZ,
    \end{displaymath}
    and there is a $\mathbf{Q} \subseteq \bX\cup\bY\cup\bZ\setminus\{X, Y\}$ such that 
    \begin{displaymath}
    \bZ\subseteq \mathbf{Q} \text{ and } X\perp_{G-(X\to Y)} Y\mid \mathbf{Q}.
    \end{displaymath}
\end{definition}

\section{OMITTED PROOFS}
\label{sec:proofs}
\begin{proof}[Proof for \cref{prop:continuity_mi}]
    Let $\bX, \bY, \bZ, \bW\subseteq\bV$ be disjoint and $\epsilon, \delta > 0$.
    Property 1) follows immediately from the definition of the conditional mutual information.
    For 2), suppose we have $I(\bX:\bY\cup \bW\mid \bZ) \le \epsilon$.
    This means
    \begin{displaymath}
        \epsilon \ge I(\bX:\bY\cup \bW\mid \bZ) = I(\bX:\bY\mid \bZ\cup \bW) + I(\bX:\bW\mid \bZ) \ge I(\bX:\bW\mid \bZ), 
    \end{displaymath}
    where the equality follows from the chain rule of conditional mutual information and the last inequality from the non-negativity of conditional mutual information.
    We get $\epsilon \ge I(\bX:\bY\mid \bZ)$ symmetrically.
    For 3), we can also apply the chain rule. Let $I(\bX:\bY\cup \bW\mid \bZ) \le \epsilon$. Then
    \begin{displaymath}
        \epsilon \ge I(\bX:\bY\cup \bW\mid \bZ) = I(\bX:\bY\mid \bW\cup \bZ) + I(\bX:\bW\mid \bZ) \ge I(\bX:\bY\mid \bW\cup \bZ)
    \end{displaymath}
    and analogously $I(\bX:\bW\mid \bZ\cup \bY) \le \epsilon$.
    
    For 4), let $I(\bX:\bY\mid \bZ) \le \epsilon$ and $I(\bX:\bW\mid \bZ\cup \bY) \le \epsilon$.
    Then we have
    \begin{displaymath}
        I(\bX:\bY\cup \bW\mid \bZ) = I(\bX:\bW\mid \bZ\cup \bY) + I(\bX:\bY\mid \bZ) \le 2\epsilon.
    \end{displaymath}

  Finally, for 5), let $I(\bX:\bY \mid \bZ\cup \bW) \le \epsilon$, $I(\bX:\bW \mid \bZ\cup \bY) \le \epsilon$ and $I(\bW:\bY \mid \bZ) \le \delta$. We apply the chain rule to the quantity $I(\bX\cup\bW : \bY \mid \bZ)$ in two ways:
  \begin{displaymath}
  I(\bX\cup\bW:\bY \mid \bZ) = I(\bW:\bY\mid \bZ) + \underbrace{I(\bX:\bY\mid \bZ\cup \bW)}_{\le \epsilon} = I(\bX:\bY\mid \bZ) + \underbrace{I(\bW:\bY\mid \bZ\cup \bX)}_{\ge 0}.
  \end{displaymath}

  Rearranging the equality and dropping the non-negative term:
  \begin{displaymath}
  I(\bX:\bY\mid \bZ) \le I(\bW:\bY\mid \bZ) + \epsilon \le \delta + \epsilon.
  \end{displaymath}

  Now apply the chain rule to $I(\bX: \bY\cup \bW \mid \bZ)$:
  \begin{displaymath}
  I(\bX:\bY\cup \bW\mid \bZ) = I(\bX:\bY\mid \bZ) + \underbrace{I(\bX:\bW\mid \bZ\cup \bY)}_{\le\epsilon} \le \delta + 2\epsilon.
  \end{displaymath}

\end{proof}

\begin{lemma}
\label{lem:algo_give_causal_input_list}
    Let $M$ be a Graphoid independence model over a set of nodes $\bV$, and $\pi:\bV\to \N$ be an ordering of $\bV$.
	Let $L_\pi$ be given as follows:  for $X, Y, \bZ \in \CI(\bV)$ we have $\CI(X, Y\mid \bZ) \in L_\pi$ iff
	\begin{displaymath}
		\pi(X) < \pi(Y) \text{ and } \bZ = \{Z\in\bV :\: \pi(X) \neq \pi(Z) < \pi(Y)\}. 
    \end{displaymath}
    For $X\in\bV$ let $\bV_X = \{V\in \bV\ : \pi(V) < \pi(X)\}$, and $\mathbf{P}_X$ be the smallest subset of $\bV_X$ such that $X\ind \bV_X\setminus \mathbf{P}_X \mid \mathbf{P}_X$ if such a set exists and else $\mathbf{P}_X = \bV_X$.  
    Then the statements 
    \begin{displaymath}
    X\ind \bV_X\setminus \mathbf{P}_X \mid \mathbf{P}_X 
    \end{displaymath}
    follow via Graphoid axioms for all $X\in \bV$ from $L_\pi$ if such a set exists.
    Further, $\mathbf{P}_X$ contains exactly the nodes $Y\in\bV_X$ with $(X\not\ind Y\mid \bV_{X} \setminus\{Y\})\in L_\pi$. I.e. the statements from the \emph{causal input list} \citep{bouckaert1995bayesian} follow from $L_\pi$.
\end{lemma}
\begin{proof}[Proof for \cref{lem:algo_give_causal_input_list}]
    We prove this statement by induction.
    For $|\bV| = 2$ and w.l.o.g. $\bV = \{X_1, X_2\}$ the statement holds trivially, as $L$ contains  $X_1\ind X_2 \mid \emptyset$ iff they are independent.
    Let $|\bV| = k\in \N_{>2}$ and the statement be true for $k-1$.
    Then the statement holds for any $X\in\bV$ with $\pi(X) < k-1$, since $\bV_X$ already follows from the tests that do not include the last node.
    
    Let $X_k$ be the last node in the ordering, and     \begin{displaymath}
    \mathbf{Q}_k = \{V\in \bV_{X_k} : V\not\ind X_k\mid \bV_{X_k} \setminus \{V\}\},
    \end{displaymath}
    i.e., our candidate set for $\mathbf{P}_{X_k}$. We first want to see that 
    \begin{displaymath}
        X_k\ind \bV_{X_k}\setminus \mathbf{Q}_k \mid \mathbf{Q}_k.
    \end{displaymath}
    If $\mathbf{Q}_k = \bV_{X_k}$ this holds trivially and with $|\bV_{X_k}\setminus\mathbf{Q}_k|=1$ the required statement is contained in $L_\pi$.
    So let $X_i, X_j\in \bV_{X_k}$ be two distinct nodes such that $(X_k\ind X_i\mid \bV_{X_k} \setminus \{X_i\})\in L_\pi$ and $(X_k\ind X_j\mid \bV_{X_k} \setminus \{X_j\})\in L_\pi$.
    By application of Graphoid axiom number 5, we get
    \begin{displaymath}
        X_k\ind X_i\mid \bV_{X_k} \setminus \{X_i\})\ \land\ X_k\ind X_j\mid \bV_{X_k} \setminus \{X_j\} \implies X_k\ind \{X_i, X_j\}\mid \bV_{X_k} \setminus \{X_i, X_j\}.
    \end{displaymath}
    Suppose there is a third $X_l\in \bV_{X_k}\setminus\mathbf{Q}_k$.
    Analogously we get
    \begin{align*}
        &X_k\ind X_l\mid \bV_{X_k} \setminus \{X_l\})\ \land\  X_k\ind \{X_i, X_j\}\mid \bV_{X_k} \setminus \{X_i, X_j\} \\
        &\implies X_k\ind \{X_i, X_j, X_l\}\mid \bV_{X_k} \setminus \{X_i, X_j, X_l\}.
    \end{align*}
    This can be repeated until we get the required statement.

    It remains to show that there is no smaller set with the same property.
    Suppose for a contradiction there is a set $\mathbf{Q}_k'$ with 
    \begin{displaymath}
        X_k\ind \bV_{X_k}\setminus \mathbf{Q}_k' \mid \mathbf{Q}'_k.
    \end{displaymath}
    Then there is at least one node $X_i\in \mathbf{Q}_k\setminus\mathbf{Q}_k'$. 
    Then we can rewrite the statement above and apply Graphoid axiom 3 to get
    \begin{displaymath}
        X_k\ind X_i \cup (\bV_{X_k}\setminus \mathbf{Q}_k') \setminus \{X_i\} \mid \mathbf{Q}'_k \implies X_k\ind X_i  \mid \bV_{X_k}\setminus\{X_i\}.
    \end{displaymath}
    But by construction, $\mathbf{Q}_k$ contains all nodes $V$ that are dependent on $X_k$ given $\bV_{X_k}\setminus\{V\}$.
    So this is a contradiction.
\end{proof}

\begin{proof}[Proof for \cref{prop:protocol_graph_markovian}]
Let $M$ be a Graphoid independence model 
    over a set of nodes $\bV$, and $\pi:\bV\to \N$ be an ordering of $\bV$.
	Let $L_\pi$ be the list of tests from the statement, i.e.  for $X, Y, \bZ \in \CI(\bV)$ we have $\CI(X, Y\mid \bZ) \in L_\pi$ iff $\pi(X) < \pi(Y)$ and
	\begin{displaymath}
		\bZ = \{Z\in\bV :\: \pi(X) \neq \pi(Z) < \pi(Y)\}. 
	\end{displaymath}
	Let $G(\pi)$ be the DAG over $\bV$ with an edge from $X\in\bV$ to $Y\in \bV$ iff $(X\not\ind Y \mid \bZ) \in L_\pi$ for some $\bZ\subseteq \bV$.

By \cref{lem:algo_give_causal_input_list} the set $L_\pi$ from \cref{prop:protocol_graph_markovian} gives 
    \begin{displaymath}
    X\ind \bV_X\setminus \mathbf{P}_X \mid \mathbf{P}_X
    \end{displaymath}
    for the smallest subset of $\bV_X$ such that this holds via Graphoid axioms and the graph $G(\pi)$ constructed contains an edge from every node in $\mathbf{P}_X$ to $X$.
    Then we can apply Theorem 2 and Corollary 2 from \citet{verma1990causal} to see that this graph is a minimal I-map of the underlying independence model $M$.
    By definition, this means that for all disjoint $\bX, \bY, \bZ\subseteq \bV$ with $\bX\neq\emptyset\neq\bY$ we get 
    \begin{displaymath}
        \bX\perp_{G(\pi)} \bY\mid \bZ \implies \bX\ind_{M} \bY\mid \bZ,
    \end{displaymath}
    and therefore all \emph{in}dependence statements are GA-redundant with respect to $L_\pi$.
\end{proof}

\begin{proof}[Proof for \cref{prop:markov_network_markovian}]
Let $M$ be a Graphoid independence model 
    over a set of nodes $\bV$.
	Let $L$ be the set of CI-statements
	\begin{displaymath}
		L = \{\CI(X, Y\,|\, \bV\setminus\{X, Y\}) :\: X, Y\in \bV,\, X\neq Y\}.
    \end{displaymath}
	Let $G(L)$ be the undirected graph over $\bV$ with an edge between $X, Y\in \bV$ iff $(X\not\ind Y \mid \bV\setminus\{X, Y\}) \in L$.

    By direct application of Theorem 12 by \citet{geiger1993logical}, we get 
    \begin{displaymath}
        \bX\perp_{G(L)} \bY\mid \bZ \implies \bX\ind_{M} \bY\mid \bZ,
    \end{displaymath}
    i.e. that all \emph{in}dependences are GA-redundant with respect to $L$.
\end{proof}

\begin{proof}[Proof for \cref{prop:sufficient_cond_non_graphoid}]
\todo{Add visualization}
Let $M$ be the independence model of a positive distribution, $L$ be a set of CI-statements and $s:=(X, Y, \bZ) \in \CI(\bV)$ with $(X\not\ind Y\mid \bZ)\not\in L$.
	Let $G$ be a graphical model 
    such that $G$ is Markovian to $L$ and $X\not\perp_{G} Y\mid \bZ$.
	We will prove the statement by constructing a Graphoid independence model that contains all CI-statements in $L$ but not $X\not\ind Y \mid \bZ$ by building a distribution that is Markovian but not faithful to a DAG.
	
	If $X$ and $Y$ are only $d$-connected given $\bZ$ by a direct edge, we set $G' = (V, E\setminus\{(X, Y), (Y, X)\})$. 
	
	Otherwise, let $N(X\sim Y\mid \bZ)$ be the set of nodes on the active paths between $X$ and $Y$ given $\bZ$ (without $X$ and $Y$).
	We construct a new graph $G' = (V', E')$ with more nodes and more edges.
    See \cref{fig:modified_graph} for an example.
	First, add all nodes in $V\setminus N(X\sim Y\mid \bZ)$ to $V'$.
	For each $W\in N(X\sim Y\mid \bZ)$, we add nodes $W_1, W_2$ to $V'$.
	For all edges $W \to W'$ with $W, W' \in N(X\sim Y\mid \bZ)$, add the edges $W_1 \to W_1'$ and $W_2\to W_2'$.
	If $W'\in V\setminus N(X\sim Y\mid \bZ)$ add $W_1\to W'$ and $W_2\to W'$. Analogously add $W\to W'_1$ and $W\to W'_2$ if $W\in V\setminus N(X\sim Y\mid \bZ)$.
	Finally add $X\to W_1$ and $X\leftarrow W_1$ whenever $X\to W$ or $X\leftarrow W$ for $W\in  N(X\sim Y\mid \bZ)$ and $Y\to W_2$ and $Y\leftarrow W_2$, respectively.
	
	Clearly, in this graph we have $X\perp_{G'} Y\mid \bZ'$, where $\bZ'$ contains all nodes in $\bZ\setminus N(X\sim Y\mid \bZ)$ and $W_1, W_2$ for $W\in \bZ\cap N(X\sim Y\mid \bZ)$.
    \philipp{Need to assume $M$ is generated by distribution!}
    Let $P$ be a distribution with $M$ as independence model (so $G$ is Markovian to $P$).
    We can now construct a distribution $P'$ as an intermediate step towards a distribution $P''$ that is Markovian and faithful to $L$.
    If $D\in \bV$ has parents in the set of copies of $N(X\sim Y\mid \bZ)$, we copy the dependence to the original node.
    For other parents, we also keep the dependence the same. 
    More formally, let $E^1, \dots, E^n$ be the $n\in \N$ parents of $D$ that are in $\bV\setminus N(X\sim Y\mid \bZ)$, $F^1, \dots, F^m$ be the $m\in \N$ parents of $D$ that are copies of nodes in $N(X\sim Y\mid \bZ)$ and $F^1_0, \dots, F^m_0$ be the corresponding original nodes in $N(X\sim Y\mid \bZ)$.
    In other words, for discrete distributions we set 
    \begin{displaymath}
        P'(D=d\mid \PA(D)=\pa(D)) := P(D=d\mid E^1=e^1, \dots, E^n=e^n, F^1_0 = f^1, \dots, F^m_0=f^m)
    \end{displaymath}
    and analogously for continuous distributions.
    If we now construct a distribution $P''$ by considering all copies $W_1, W_2, \dots$ for $W\in N(X\sim Y\mid \bZ)$ as a single vector-valued variable (and keep the canonical mapping between node names $W=(W_1, W_2, \dots)$), this distribution, again, contains all CI-statements we had in $L$ but $X\ind_{\hspace{-.4em}P''} Y\mid \bZ$, as we will argue next.
	Further, it is also positive if $P$ was and therefore its independence model is Graphoid.

    \begin{figure}[htp]
		\centering
		\begin{tikzpicture}
			\node[obs] (X) {$X$} ;
			\node[obs, right=of X] (V1) {$V_1$} edge[-] (X) ;
            \node[obs, right=of V1] (W1) {$W_1$} edge[-] (V1) ;
            \node[obs, below =of V1, yshift=1cm] (V2) {$V_2$} ;
            \node[obs, right=of V2] (W2) {$W_2$} edge[-] (V2) ;
            \node[obs, right=of W2] (Y) {$Y$} edge[-] (W2) ;

         \begin{scope}[on background layer]
    		\node[draw, ellipse, dashed, fit=(V1) (V2), inner sep=0.05cm] (V_group) {};
       \end{scope}

        \begin{scope}[on background layer]
    		\node[draw, ellipse, dashed, fit=(W1) (W2), inner sep=0.05cm] (W_group) {};
	       \end{scope}
			
		\end{tikzpicture}
		\caption{Example for a modified graph $G'$ from the proof of \cref{prop:sufficient_cond_non_graphoid}. All intermediate nodes on the path $X - V - W - Y$ are replaced by two copies and each new path only connects to either $X$ or $Y$.}
		\label{fig:modified_graph}
	\end{figure}

	To see the former, let $(A\ind B\mid \bC)\in L$.
    In our construction we did not add any dependences.
    If $A$ and $B$ were not connected given $\bC$ in $G$, they still are not, as we did not add an edge between nodes that were disconnected and we did not introduce any colliders between these nodes.
    If they were connected, but still independent in $P$, they also still are, since each Markov kernel of $P'$ is defined using $P$ and a subset of the parents from $G$.
    So overall we get $A\ind_{\hspace{-.4em}P''} B\mid \bC$.
\philipp{Need to make this more formal? Say Markov Kernel in $P''$ is the product of Kernels in $P'$? Then it's quite obvious that facotrization still holds. Also makes non-negativity clearer.}
    
	Now suppose $(A\not\ind B\mid \bC)\in L$, i.e., there is an active path between $A$ and $B$ given $\bC$ (as otherwise $G$ would not be Markovian to $L$).
	By assumption, $A$ and $B$ are not coupled over $s$ given $\bC$. 
    This means there is an $s$-active path between them, which entails that this path does not contain any subpath from $X$ to $Y$ that is active given $\bZ$.
	As we only modified the paths between $X$ and $Y$ that are active given $\bZ$, the path between $A$ and $B$ is still active.
    So we also still have $A\not\ind_{\hspace{-.4em}P''} B\mid \bC$.
	
	It now suffices to note that every positive distribution (and therefore particularly the one we just constructed) implies a Graphoid independence model. 
\end{proof}

\begin{proof}[Proof for \cref{lem:tree_tests_non_graphoid}]
    Let $L = \{\CI(X, Y\mid Z) : X, Y, Z\in V, X\neq Y, X\neq Z, Y\neq Z\}$ such that $\CI(X, Y\mid Z) = \CI(Y, X\mid Z)$ for all distinct $X, Y, Z\in V$. 
    Then for any CI-statement $\CI(X, Y\mid Z)\in L$, only axiom 1 from \cref{def:graphoid} is applicable, since all the other axioms require at least one of the operands to be a set of size larger than one, which we don't have in $L$.
    By assumption, we have no contradictions between the statements in $L$ and the statements that follow from axiom 1.
    And also $L$ contains all statements that can be derived with axiom 1.
\end{proof}

\begin{proof}[Proof for \cref{prop:error_correction_trees}]
    Let $S = \{(X, Y, Z) : X, Y, Z\in V, X\neq Y, X\neq Z, Y\neq Z\}$.
	Further let $T^*\in \cT_n$ and let $M$ be an independence model with $\MD_S(T^*, M) \le \lfloor (n-1) / 2 \rfloor$ for $n\in \N_{> 3}$. 
    To arrive at a contradiction we assume
	\begin{equation}
		\MD_S(T', M) < \MD_S(T^*, M), \label{eq:proof_tree_correction}
	\end{equation}
    for some tree $T'\in\cT_n$ with $T'\neq T^*$.
    Since the graphs differ and both are trees, there are nodes $X, Y$ that were connected by a direct edge in $T^*$ but are not in $T'$.
    But since they are spanning trees there still exists a path of length $k\in \N_{>1}$ between $X$ and $Y$ in $T'$.
    So there is at least one node $Z$ such that we have $X\perp_{T'} Y\mid Z$ but $X\not\perp_{T^*} Y\mid Z$, because of the edge between $X$ and $Y$ in $T^*$.
    So we have at least one CI-statement difference between $T'$ and $T^*$.
    For the second difference, there are two cases how $Z$ is connected to $X$ in $T^*$.
    \begin{enumerate}
        \item If $Y$ lies between $X$ and $Z$, i.e. $X-Y\overset{*}{-}Z$ we get $X\not\perp_{T'}Z\mid Y$ but $X\perp_{T^*}Z\mid Y$.
        \item If $X$ lies between $Y$ and $Z$, i.e. $Z\overset{*}{-}X-Y$ we get $Z\not\perp_{T'} Y
        \mid X$ but $Z\perp_{T^*} Y\mid X$.
    \end{enumerate}
    So in any case we have another difference between $T^*$ and $T'$.

    W.l.o.g. we can assume that $Z$ is the node closest to $X$ on the path between $X$ and $Y$, i.e. the node that has a direct edge to $X$.
    Now let $W$ be another node (which exists, because $n>3$).
    There are three cases how this node can be connected to $X$ along the unique path in $T'$.
    \begin{enumerate}
        \item If $Y$ is between $Z$ and $W$, i.e. if we have 
        \begin{displaymath}
            X - Z \overset{*}{-} Y \overset{*}{-} W.
        \end{displaymath}
        Then we have $X\perp_{T'}W\mid Z$ but $X\not\perp_{T^*} W\mid Z$, since in $T^*$ we had the direct edge from $X$ to $Y$.
        \item If $Y$ does not lie on the paths between $X$ and $W$ and $X$ not on the one between $Y$ and $W$, i.e. if we have the structure in \cref{fig:proof_error_correction_tree}, we will further subdivide this into two subcases depending on $T^*$.
        \begin{enumerate}
            \item In $T^*$ we have $Y$ between $X$ and $W$, i.e. $X-Y\overset{*}{-}W$. Then $X\not\perp_{T'}W\mid Y$ but $X\perp_{T^*}W\mid Y$.
            \item In $T^*$ we have $X$ between $Y$ and $W$, i.e. $W\overset{*}{-}X-Y$. Then we have $Y\not\perp_{T'}W\mid X$ but $Y\perp_{T^*}W\mid X$.
        \end{enumerate}
        \item If $X$ is between $W$ and $Y$, i.e.
        \begin{displaymath}
            W \overset{*}{-} X - Z \overset{*}{-} Y.
        \end{displaymath}
        Then this case works symmetrically to the first one.
    \end{enumerate}
    Since we chose $W$ arbitrarily, we can repeat this for every node other than $X, Y, Z$.
    This means, we get $n-3$ more contradictions between the independence model of $T^*$ and $T'$.
    So overall we have $n-1$ CI-statements where $T'$ and $T^*$ disagree.
    So even if $T'$ fits all the $\lfloor (n-1)/2 \rfloor$ tests where $T^*$ and $M$ differ, it will also differ by another $\lceil (n-1)/2 \rceil$ statements from $T^*$ and therefore also from $M$.
    This is a contradiction to \cref{eq:proof_tree_correction}.
    
    \begin{figure}[htp]
		\centering
		\begin{tikzpicture}
			\node[obs] (X) {$X$} ;
			\node[obs, right=of X] (Z) {$Z$} edge[-] node[above]{} (X) ;
            \node[vardashed, right =of Z] (empty) {} edge[-] node[above]{*} (Z);
			\node[obs, right =of empty] (Y) {$Y$} edge[-] node[above]{*} (empty);
            \node[obs, below =of empty] (W) {$W$} edge[-] node[right]{*} (empty);
			
		\end{tikzpicture}
		\caption{Visualisation how the node $W$ from the proof of \cref{prop:error_correction_trees} can be connected to $X$ and $Y$ in case 2 of the proof.}
		\label{fig:proof_error_correction_tree}
	\end{figure}
\end{proof}

\begin{proof}[Proof for \cref{lem:sp}]
    Let $G^*$ be a DAG and consider some independence model $M$ (that is not necessarily Markovian and faithful to $G^*$).
    Concerning (1): Suppose there is no permutation $\pi^*$ such that all CI-statements in $L_{\pi^*}$ are as in $M_{G^*}$.
    So for every $\pi$ there is a CI-statement $\CI(X, Y\mid \bZ)$ that is not as in $M_{G^*}$. 
    If $X\perp_{G^*} Y\mid \bZ$, then $G^*$ does not have an edge between $X$ and $Y$. But due to $X\not\ind Y\mid \bZ$ the graph $G(\pi)$ would have this edge. Analogously, if $X\not\perp_{G^*} Y\mid \bZ$ then $G(\pi)$ would not have this edge. Since this holds for all permutations, SP won't output $G^*$.

    Concerning (2): Suppose there is a permutation $\pi$ with less than $|G^*|$ dependences.
    But since SP only adds an edge to $G(\pi)$ for every dependence, $|G(\pi)| < |G^*|$ and SP does not output $G^*$ as it only outputs the sparsest graphs.
\end{proof}

\subsection{Omitted derivations of smaller claims}
\label{sec:small_claims}

\paragraph{\Cref{ex:motivation}}
\emph{\enquote{$[\dots]$ if we have $X_1\not\ind Y$ and $X_1\ind X_2$ we also have $X_1
    \not\ind Y\mid X_2$}}. Suppose we have $X_1\not\ind Y$, and $X_1\ind X_2$. 
    For a contradiction assume  $X_1\ind Y\mid X_2$ holds.
    Then we can apply Graphoid axiom 4 (contraction) and get
    \begin{displaymath}
        X_1\ind X_2\,  \land\, X_1\ind Y\mid X_2 \implies X_1\ind \{X_2, Y\}. 
    \end{displaymath}
    By applying axiom 2 (decomposition) we get $X_1\ind Y$. 
    This is a contradiction to our assumptions.

\paragraph{\Cref{ex:graphoid_hurts_correction}}
\emph{\enquote{$[\dots]$ note that the marginal tests already imply $X\not\ind Y\mid Z$ and $Y\not\ind Z\mid X$}}. 
Suppose we have $X\not\ind Y, X\ind Z$, and $Y\not\ind Z$.
Assume for a contradiction that we have $X\ind Y\mid Z$.
Like before, we can apply Graphoid axiom 4 and get
    \begin{displaymath}
        X\ind Z \, \land\, X\ind Y\mid Z \implies X \ind \{Y, Z\}. 
    \end{displaymath}
    By applying axiom 2 (decomposition), we get $X \ind Y$, which is a contradiction.
    If we assume $Y\ind Z\mid X$, axioms 4 and 1 (symmetry) also gives us
    \begin{displaymath}
        Z\ind X \, \land \, Z\ind Y\mid X \implies Z \ind \{X, Y\}, 
    \end{displaymath}
    which results in $Z\ind Y$ via axiom 2.
    This contradicts our assumptions.

\emph{\enquote{$[\dots]$ there are four tests in favour of the actual ground truth model. 
    But there are five that would be explained, e.g., by the graph $X\to Y\to Z$}}.
    The tests in agreement with the ground truth model are
    \begin{displaymath}
    X\not\ind Y, \quad X\not\ind Y\mid Z,\quad X\ind Z\mid Y, \quad \text{and} \quad X\ind Z.
    \end{displaymath}
    The tests that match the alternative are
    \begin{displaymath}
    X\not\ind Y,\quad X\not\ind Y\mid Z,\quad X\ind Z\mid Y,\quad Y\not\ind Z, \quad \text{and}\quad  Y\not\ind Z\mid X.
    \end{displaymath}

\philipp{Also list which tests are in favour of which graph}

\section{ADDITIONAL RELATED WORK}
\label{sec:additional_related_work}

There is also a vast literature on making the discovery of graphical models robust against statistical uncertainty.
Notable examples include \citet{kalisch2007estimating}, who propose a stronger version of faithfulness under which the PC algorithm is uniformly consistent, and \citet{bhattacharyya2021near}, who show finite-sample bounds for tree learning.
Other approaches focus on controlling the statistical error: \citet{strobl2019estimating,li2009controlling} use techniques from multi-hypothesis testing to control the error rate of edges. 
Robustness towards violations of parametric assumptions of the PC algorithm is studied by \citet{kalisch2008robustification, harris2013pc}.
Additional strategies include \citet{wienobst2020recovering,kocaoglu2023characterization}, who propose to use a subset of tests that is assumed to be statistically more robust and investigate which graphical structures can be identified by them. 
\citet{ramsey2016improving} proposes to always pick the separating sets with the highest $p$-value, and \citet{li2019constraint} choose them such that the separating sets are consistent with the final graph.
Finally, \citet{rohekar2021iterative} proposes an algorithm that is anytime valid.

\section{GRAPHS AS ERROR-CORRECTING CODES}
\label{subsec:noisy_channel}
To motivate why we call some conditional independence tests \enquote{redundant}, we will phrase graph discovery as a coding problem.
Suppose a sender picks a graph $G$ from a set of graphs. 
Since the Markov equivalence class of this graph is identified by a sequence of CI-statements, she can encode this equivalence class in a binary string $s\in \{0, 1\}^k$ for some $k\in \N$, where each bit represents whether a certain CI-statement holds or not.
If a receiver knows the sequence of CI-statements, she can perfectly recover the equivalence class of $G$ from $s$.
In this scenario, the mapping $G\mapsto s$ is a coding scheme.
Then a (deterministic) CI-based discovery algorithm implicitly defines a decoding scheme $s \mapsto G$.
Unfortunately, the discovery algorithm rarely receives $s$, but a noisy version of it, since the CI-tests can have erroneous outputs.
If we assume, for now, that the errors of each bit are independent, Shannon's noisy-channel coding theorem \citep{mackay2003information} asserts that there is a coding scheme 
such that messages can be transmitted with arbitrarily small error probability as the number of sent bits approaches infinity.\footnote{It is worth noting that the relationship to the noisy-channel coding theorem is just an analogy and the theorem cannot be applied to our setting. 
The reason is that the theorem holds when the number of sent bits approaches infinity, while we can only conduct finitely many CI-tests.}
This begs the question of what it would mean to have these \emph{redundant} bits added to $s$.
In a (literal) noisy channel, one could resend bits, but clearly redoing a CI-test does not give us additional information. 
Also, techniques like bootstrapping or cross-validation cannot help if the errors come from faithfulness violations.
The main goal of this work is to address the question of which tests are suitable to add redundancy to our encoding.

\section{GRAPHICAL CRITERION FOR GA-REDUNDANT DEPENDENCES}
\label{sec:graphical_crit_dependences}
Building on insights from \citet{bouckaert1995bayesian}, we will now show a graphical criterion for dependence statements that follow in the situation of \cref{prop:markov_network_markovian,prop:protocol_graph_markovian}.
\begin{corollary}[\citet{bouckaert1995bayesian} Thm. 3.6, 3.10]
	\label{prop:coupling}
    Let $M$ be a Graphoid independence model over $\bV$, and $G$ be a graph constructed like in \cref{prop:markov_network_markovian,prop:protocol_graph_markovian}.
    Let $\bX, \bY, \bZ
    \subseteq \bV$ be disjoint with $\bX\neq\emptyset\neq \bY$.
    If $\bX, \bY$ are  coupled in $G$ given $\bZ$, then $\bX\not\ind\bY\mid \bZ$ is GA-redundant.
\end{corollary}

\begin{proof}[Proof for \cref{prop:coupling}]

    Let us first consider the case of DAGs.
    By application of \cref{lem:algo_give_causal_input_list}, Theorem 2, and Corollary 2 from \citet{verma1990causal}, we get that the graph $G(\pi)$ from \cref{prop:protocol_graph_markovian} is a minimal I-map of the independence model $M$.
    Let $\bX, \bY, \bZ \subseteq \bV$ be disjoint such that $\bX$ and $\bY$ are coupled given $\bZ$ in $G(\pi)$.
    By Theorem 3.10 by \citet{bouckaert1995bayesian} we get 
    \begin{displaymath}
        \bX\not\ind \bY\mid \bZ,
    \end{displaymath}
    i.e. that this dependence holds in every independence model that contains $L_\pi$.
    This means the dependence is GA-redundant with respect to the class of DAGs and $L_\pi$.
    
    Similarly, for undirected graphs we also get that $G(L)$ is a minimal I-map from Theorem 12 by \citet{geiger1993logical}.
    Then we can apply Theorem 3.6 from \citet{bouckaert1995bayesian} to get 
    \begin{displaymath}
        \bX\not\ind \bY\mid \bZ
    \end{displaymath}
    again, i.e. whenever $L$ is contained in the independence model we get the dependence.
\end{proof}

Further, this criterion is not only necessary but also sufficient in the scenario of \cref{prop:markov_network_markovian,prop:protocol_graph_markovian}.
\begin{corollary}[coupling is necessary and sufficient]
    \label{cor:coupling_necessary_and_sufficient}
    Let $M$ be a Graphoid independence model over $\bV$, and $G(\pi)$ be a graph constructed like in \cref{prop:protocol_graph_markovian} and
    $(X, Y, \bZ) \in \CI(\bV)$.
    Then $X, Y$ are coupled in $G(\pi)$ given $\bZ$ if and only if $X\not\ind Y\mid \bZ$ is GA-redundant with respect to $L_\pi$.
    
    Now let $G(L)$ be a graph constructed like in \cref{prop:markov_network_markovian}.
    Let $\bX, \bY, \bZ
    \subseteq \bV$ be disjoint with $\bX\neq\emptyset\neq \bY$.
    Then $\bX, \bY$ are  coupled in $G(L)$ given $\bZ$ if and only if $\bX\not\ind\bY\mid \bZ$ is GA-redundant with respect to $L$. 
\end{corollary}
\begin{proof}[Proof for \cref{cor:coupling_necessary_and_sufficient}]
    Similarly to before we can use \cref{lem:algo_give_causal_input_list} to get the statements 
    \begin{displaymath}
    X\ind \bV_X\setminus \mathbf{P}_X \mid \mathbf{P}_X,
    \end{displaymath}
    for every $X\in\bV$, where $\bV_X = \{V\in \bV\ : \pi(V) < \pi(X)\}$, and $\mathbf{P}_X$ is the smallest subset of $\bV_X$ such that $X\ind \bV_X\setminus \mathbf{P}_X \mid \mathbf{P}_X$ if such a set exists and else $\mathbf{P}_X = \bV_X$.  
    This also entails $X\not\ind Y\mid \mathbf{P}_X\setminus \{Y\}$ for all $Y\in \mathbf{P}_X$, which can be seen by constructing the following contradiction:
    suppose we have $X\ind Y\mid \mathbf{P}_X\setminus\{Y\}$.
    Together with the \emph{in}dependence above, Graphoid axiom 4 implies
    \begin{displaymath}
        X\ind \{Y\}\cup \bV_X \setminus \mathbf{P}_X \mid \mathbf{P}_X\setminus\{Y\}
        = X\ind \bV_X \setminus (\mathbf{P}_X \setminus \{Y\}) \mid \mathbf{P}_X\setminus\{Y\},
    \end{displaymath}
    which is a contradiction to the minimality of $\mathbf{P}_X$.
    In other words, we have the \emph{dependency base} \citep{bouckaert1995bayesian}.
    The first statement then follows from Theorem 3.11 from \citet{bouckaert1995bayesian}.

    Also like before, we get that $G(L)$ is a minimal I-map from Theorem 12 by \citet{geiger1993logical}.
    Then the second statement follows by applying Theorem 3.6 from \citet{bouckaert1995bayesian}. 
\end{proof}
Yet, this does not hold anymore if we consider multiple additional CI-statements, as the following example shows.
\begin{example}[iterated coupling]
    \label{ex:iterated_coupling}
    Consider an independence model that is Markovian and faithful to the graph in \cref{fig:iterating_coupling_fails}.
    Suppose we want to learn the graph with the method from \cref{prop:protocol_graph_markovian} with a valid causal ordering $\pi$.
    Now suppose we conduct the test $\CI(X, W\mid Z)$ afterwards.
    $X$ and $W$ are not coupled given $Z$ in $G(\pi)$, since there is no edge between $X$ and $W$.
    But now the tests in $L_\pi\cup\{X\not\ind W\mid Z\}$ imply $X\not\ind Y\mid Z$.
    This can be seen as follows:
    from \cref{prop:protocol_graph_markovian} we get $X\ind W\mid \{X, Y\}$, since an \emph{in}dependence follows from $L_\pi$ if the nodes are separated in the graph.
    We will now derive a contradiction.
    Suppose we would have $X\ind Y\mid Z$.
    Then we could apply Graphoid axiom 4 to $X\ind W\mid \{Z, Y\}$ and $X\ind Y\mid Z$ to get $X \ind \{Y, W\} \mid Z$.
    Using axiom 3 we get $X\ind W\mid Z$.
    But we have conducted the respective test and found $X\not\ind W\mid Z$.
    So $X\not\ind Y\mid Z$ follows.
    But $X$ and $Y$ are not coupled given $Z$ in $G(\pi)$.
    This means the dependence only follows from $L_\pi\cup\{W\not\ind X\mid Z\}$, not from $L_\pi$ alone.
    Particularly, coupling is not a necessary criterion for GA-redundancy anymore if we conduct more tests than the ones in $L_\pi$.
   \begin{figure}[htp]
		\centering
		\begin{tikzpicture}
			\node[obs] (X) {$X$} ;
			\node[obs, right=of X] (Z) {$Z$} edge[<-] (X) ;
            \node[obs, right=of Z] (Y) {$Y$} edge[->] (Z) ;
            \node[obs, right=of Y] (W) {$W$} edge[->] (Y) ;
			
		\end{tikzpicture}
		\caption{In the setting of \cref{prop:protocol_graph_markovian}, $X\not\ind Y\mid Z$ only follows from $L_\pi\cup\{X\not\ind W\mid Z\}$, not from $L_\pi$ alone.}
		\label{fig:iterating_coupling_fails}
	\end{figure}
\end{example}
This highlights the utility of \cref{prop:sufficient_cond_non_graphoid} and especially \cref{cor:iterated_sufficient_redundancy}, which we will see later.
If we want to conduct multiple redundant tests, a sufficient criterion for GA-redundancy is not enough to exclude all GA-redundant tests.
On the other hand, the criterion in \cref{prop:sufficient_cond_non_graphoid} may miss some PGM-redundancies, yet all tests from this criterion are guaranteed to be PGM-redundant.

\section{ADDITIONAL REMARKS AND EXAMPLES}
\label{sec:additional_examples}
\subsection{Examples}
\begin{example}[Wrong Collider-Structure]
	\label{ex:wrong_v_structure}
	In this example, we assume we want to find a graph using the PC algorithm \cite{pearl2009causality}.
	Consider the family of graphs depicted in \cref{fig:wrong_v_structure_gt}.
	Assume there are nodes $Z$ and $X_i$ for $i=0, \dots, k\in\N_{>1}$.
	Further assume all CI-test results are Markovian and faithful to the graph, except for 
	\begin{displaymath}
		X_0\ind Z.
	\end{displaymath}
	Then the algorithm would wrongly detect a collider structure between $X_0, X_1$ and $Z$, as $X_1$ is not in the separating set of $X_0$ and $Z$.
	Then, by application of Meek rules \citep{meek1995causal}, this orientation propagates along the path.
	As in the true graph, there is no collider between $X_2, X_1$ and $X_0$, the triplet will not be oriented as collider, and by Meek rule R1 the edge $X_1-X_2$ will become $X_1\to X_2$.
	The same argument holds for all further $X_{i+1}, X_i$ with $i>2$, which will cause the algorithm to orient the whole chain of $X_i$ the wrong way around, as visualised in \cref{fig:wrong_v_structure_res}.
	In other words, a single false CI-statement might cause $k+1$ wrongly directed edges.

    If the purpose of the graphical model is simply a concise representation of the CI-statements in the data, this might be acceptable.
    But if the model is interpreted, e.g., as a causal model and it is used in some downstream task, this might be worrisome.
    Although graph discovery is already algorithmically expensive, in such cases, a practitioner might be willing to incur an additional computational overhead to make the results more robust.
	\begin{figure}[htp]
		\centering
        \subfloat[Ground truth graph]{
        \label{fig:wrong_v_structure_gt}
		\begin{tikzpicture}
			\node[obs] (X0) {$X_0$} ;
			\node[obs, right=of X0] (Z0) {$Z$} edge[->] (X0) ;
			\node[obs, left=of X0] (X1) {$X_1$} edge[->] (X0) ;
			\node[obs, left =of X1] (X2) {$X_2$} edge[->] (X1) ;
			\node[obs, left =of X2] (X3) {$X_3$} edge[->] (X2) ;
			\node[left=of X3] (dots) {$\hdots$} edge[->](X3);
            \draw (X1) edge[->, bend left=20]  (Z0);
		\end{tikzpicture}
        }
        
        \subfloat[Wrongly inferred graph]{
        \label{fig:wrong_v_structure_res}
		\begin{tikzpicture}
			\node[obs] (X0) {$X_0$} ;
			\node[obs, right=of X0] (Z0) {$Z$};
			\node[obs, left=of X0] (X1) {$X_1$} edge[<-, tabgreen] (X0) ;
			\node[obs, left =of X1] (X2) {$X_2$} edge[<-, tabred] (X1) ;
			\node[obs, left =of X2] (X3) {$X_3$} edge[<-, tabred] (X2) ;
			\node[left=of X3] (dots) {$\hdots$} edge[<-, tabred](X3);
            \draw (X1) edge[<-, tabgreen, bend left=20]  (Z0);
		\end{tikzpicture}
        }
		\caption{When the collider-structure $X_0\to X_1\leftarrow Z$ is wrongly identified, all edges $X_{i+1}\to X_i$ will be oriented the wrong way around.}
		\label{fig:graph_v_structure_wrong}
	\end{figure}
	
\end{example}

\begin{example}[P- but not GA-redundant]
    \label{ex:prob_but_not_graphoid}
    In this example, we will see a case where some CI-statements are P-redundant but not GA-redundant. Suppose we have random variables $\bV = \{X, Y, Z, W\}$ and let
    \begin{displaymath}
        L = \{X\ind Y\mid \{Z, W\},\quad X\ind Y\mid \emptyset,\quad Z\ind W\mid X,\quad Z\ind W\mid Y\}.
    \end{displaymath}
    \citet{studeny1992conditional} (Proposition 5) showed that for all probability distributions, this entails also
    \begin{displaymath}
        X\ind Y\mid Z,\quad X\ind Y\mid W,\quad Z\ind W\mid \{X, Y\},\quad  Z\ind W\mid \emptyset.
    \end{displaymath}
    But none of these statements follows from the Graphoid axioms, so they are P-redundant but not GA-redundant.
\end{example}

\begin{example}[PC results can be non-Markovian]
	\label{ex:pc_not_markov}
	Suppose we have an independence model that is Markovian and faithful to the graph in \cref{fig:pc_not_markov_gt}, except for the dependence $X_1\not\ind X_3$ and the \emph{in}dependence $X_1\ind X_3\mid Y$.
	Further, suppose we want to recover this graph with the PC algorithm.
	In the first round of the algorithm, it will conduct all marginal independence tests.
    This will give us the intermediate skeleton in \cref{fig:pc_not_markov_intermediate}.
	In the following rounds, the algorithm will conduct all tests with conditioning set of size one and two, which will result in the graph in \cref{fig:pc_not_markov_undirected}.
	But now in the orientation phase, we have conflicting evidence for the existence of colliders, and depending on how exactly the algorithm resolves them, we will get different results.
	For example, in the default implementation in \texttt{causal-learn} \citep{zheng2024causal}, the algorithm simply picks the first orientation according to the (non-predetermined) ordering in which it conducts the CI-tests.
	Suppose the algorithm first checks whether the unshielded triplet $X_1-Y-X_2$ is a collider.
    This is the case, as $Y$ is not in the separation set of $X_1$ and $X_2$, which is the empty set.
    The same holds for $X_2$ and $X_3$.
    If the algorithm sticks with these orientations, it ignores that $Y$ is indeed a member of all separating sets of $X_1$ and $X_3$.
    Therefore, it would output the graph in \cref{fig:pc_not_markov_gt}.
    But note that this graph does imply $X_1\ind X_3$, which does not hold in our independence model.
    It can be checked (e.g., with \texttt{Z3}) that the CI-tests that we used are indeed a valid Graphoid.	
	\begin{figure}[htp]
		\centering
		\subfloat[PC will output this graph, even when the independence model contains $X_1\not\ind X_3$ and $X_1\ind X_3\mid Y$. But then the graph is not Markovian to the given input.]{
				\begin{tikzpicture}
		\node[obs] (Y) {$Y$} ;
		\node[obs, above right=of Y] (X2) {$X_3$} edge[->] (Y) ;
		\node[obs, above =of Y] (X1) {$X_2$} edge[->] (Y) ;
		\node[obs, above left=of Y] (X0) {$X_1$} edge[->] (Y) ;
	\end{tikzpicture}
			\label{fig:pc_not_markov_gt}
		}
		\hspace{.5cm}
		\subfloat[Intermediate skeleton that PC finds after conducting all marginal independence tests.]{
			\begin{tikzpicture}
		\node[obs] (Y) {$Y$} ;
		\node[obs, above right=of Y] (X2) {$X_3$} edge[-] (Y) ;
		\node[obs, above =of Y] (X1) {$X_2$} edge[-] (Y) ;
		\node[obs, above left=of Y] (X0) {$X_1$} edge[-] (Y) ;
        \draw (X0) edge[-, bend left=80] (X2);
	\end{tikzpicture}
			\label{fig:pc_not_markov_intermediate}
		}
		\hspace{.5cm}
		\subfloat[Final skeleton that PC finds before orienting the edges.]{
			\begin{tikzpicture}
		\node[obs] (Y) {$Y$} ;
		\node[obs, above right=of Y] (X2) {$X_3$} edge[-] (Y) ;
		\node[obs, above =of Y] (X1) {$X_2$} edge[-] (Y) ;
		\node[obs, above left=of Y] (X0) {$X_1$} edge[-] (Y) ;
	\end{tikzpicture}
			\label{fig:pc_not_markov_undirected}
		}
		\caption{This example illustrates how the PC algorithm arrives at a graph that contradicts a dependence statement used as input to the algorithm.}
		\label{fig:pc_not_markov}
	\end{figure}
\end{example}

\subsection{Iterated sufficient criterion for PGM-redundancy}
Note that \cref{prop:sufficient_cond_non_graphoid} requires a graph $G$ to be Markovian to a set of CI-statements $L$.
If we want to find several PGM-redundant CI-statements like in the experiment shown in \cref{fig:graphoid_vs_graphical_sachs}, we can only apply \cref{prop:sufficient_cond_non_graphoid} as long as the additional tests imply dependences.
If one test returns an \emph{in}dependence and we continue to wrongly apply \cref{prop:sufficient_cond_non_graphoid}, we could end up with CI-statements that follow from the previously conducted tests, as the following example shows.
\begin{example}
        Let $G$ be the graph given in \cref{fig:iterating_suff_fails}.
        Suppose we have $L_\pi$ as in \cref{prop:protocol_graph_markovian} with respect to the ordering $X, Y, Z, W$.
        From \cref{prop:sufficient_cond_non_graphoid} we know that $X\not\ind Z\mid W$ is PGM-redundant.
        Suppose now, we conduct this additional test and actually find $X\ind Z\mid W$.
        If now we would try to read further PGM-redundant CI-statements from \cref{prop:sufficient_cond_non_graphoid}, we might find statements that follow from $L_\pi$ in conjunction with $X\ind Z\mid W$.
        To see this, note that also $X\not\ind Z\mid \emptyset$ would be PGM-redundant according to the criterion \cref{prop:sufficient_cond_non_graphoid}.
        But if we already have $X\ind Z\mid W$, also $X\ind Z\mid \emptyset$ follows.
        This can be seen as follows: 
        According to \cref{prop:protocol_graph_markovian} we can read off $X\ind W\mid Z$ from $G$.
        Then we can apply Graphoid axiom 4 and get
        \begin{displaymath}
            X\ind Z \mid \{W\}\cup \emptyset\ \land \ X\ind W\mid \{Z\}\cup \emptyset \implies X\ind \{W, Z\} \mid \emptyset,
        \end{displaymath}
        and by application of Graphoid axiom 2, we get the required statement.
        \begin{figure}[htp]
		\centering
		\begin{tikzpicture}
			\node[obs] (X) {$X$} ;
			\node[obs, right=of X] (Y) {$Y$} edge[-] (X) ;
            \node[obs, right=of Y] (Z) {$Z$} edge[-] (Y) ;
            \node[obs, right=of Z] (W) {$W$};
			
		\end{tikzpicture}
		\caption{Both $X\not\ind Z\mid W$ and $X\not\ind Z$ are PGM-redundant. But if we conduct a test and find $X\ind Z\mid W$ also $X\ind Z$ follows.}
		\label{fig:iterating_suff_fails}
	\end{figure}
\end{example}
But the following corollary states that we can continue in a similar fashion if we check separations with the graph $G'$ from the proof of \cref{prop:sufficient_cond_non_graphoid}.

\begin{corollary}[Iterated PGM-redundancy]
\label{cor:iterated_sufficient_redundancy}
Let $M$ be the independence model of a positive distribution, $L$ be a set of CI-statements and $s:=(X, Y, \bZ) \in \CI(\bV)$ with $(X\not\ind Y\mid \bZ)\not\in L$.
	Let $G$ be a graphical model 
    such that $G$ is Markovian to $L$ and $X\not\perp_{G} Y\mid \bZ$.
	Let $G'$ be a graphical model like in the proof of \cref{prop:sufficient_cond_non_graphoid}, i.e. a graphical model such that \emph{separation} of (or given) a node $W$ that has been copied is defined as the usual separation of (or given) the set of its $k\in \N$ copies $\{W_1, \dots, W_k\}$.
    Further let $G'$ be Markovian to $L$ and $X\not\perp_{G'} Y\mid \bZ$.
	If there is no $(A \not\ind B\mid \bC) \in L$ such that $A$ is coupled over $s$ with $B$ given $\bC$, then $s$ is PGM-redundant given $L$.

    Further, the graph $G'$ constructed in the proof of \cref{prop:sufficient_cond_non_graphoid} is Markovian (in terms of the above definition of separation) to $L\cup \{\lnot s\}$.
\end{corollary}
\begin{proof}[Proof for \cref{cor:iterated_sufficient_redundancy}]
    The proof works analogously to the proof of \cref{prop:sufficient_cond_non_graphoid}. 
\end{proof}

\subsection{Relationship between coupling over nodes and coupling}
The similarity in names between \cref{def:s_coupling,def:coupling} implies that these concepts are related.
Obviously, they slightly differ in their scope, as coupling is defined for three sets of variables, while decoupling is concerned with two triplets consisting of two variables and a set of variables, respectively.
In the following, we will see how these notions can coincide.
\begin{lemma}[coupling implies coupling over themselves]
\label{lem:coupling_implies_s_coupling}
    Let $G$ be a DAG or undirected graph over $\bV$. 
    If $(X, Y, \bZ)\in \CI(\bV)$ are coupled then $(X, Y, \bZ)$ are coupled over $(X, Y, \bZ)$. 
\end{lemma}
\begin{proof}
    Suppose $X$ and $Y$ are coupled given $\bZ$.
    Then by definition, there is an edge between $X$ and $Y$.
    If $G$ is a DAG, we have $X\perp_{G-(X\to Y)} Y\mid \bZ$ (since $\bX$ and $\bY$ are singleton sets and thus $\mathbf{Q} = \bZ$).
    But this means, the edge $X\to Y$ is the only active path between $X$ and $Y$ given $\bZ$.
    Analogously, if $G$ is an undirected graph, we have $\Adj(X)\subseteq \{Y\}\cup \bZ$ (or symmetrically for $\Adj(Y)$).
    Again, this means the edge $X-Y$ is the only active path given $\bZ$.
    But this path is not an $(X, Y, \bZ)$-active path, since it contains $X$ and $Y$ and is active given $\bZ$.
    So, $(X, Y, \bZ)$ are coupled over $(X, Y, \bZ)$.
    
\end{proof}

The question that remains is whether coupling and coupling over $s$ encode the same concept.
The next example shows a case where they differ.
\begin{example}[coupling and coupling over nodes differ]
        Let $G$ be the graph in \cref{fig:coupling_vs_s_coupling}.
        Then $A$ and $B$ are coupled over $(X\not\ind Y\mid \emptyset)$ given $\emptyset$, since there is one active path $(A, X, Z, Y, B)$ but also the sub-path $(X, Z, Y)$ is active given $\emptyset$.
        So in particular, if $(A\not\ind B\mid \emptyset)\in L$, we cannot conclude that $X\not\ind Y$ is PGM-redundant.
        But neither $A$ and $B$ nor $X$ and $Y$ are coupled given $\emptyset$, as they are not adjacent.
        
        \begin{figure}[htp]
		\centering
		\begin{tikzpicture}
			\node[obs] (A) {$A$} ;
			\node[obs, right=of A] (X) {$X$} edge[-] (A) ;
            \node[obs, right=of X] (V) {$Z$} edge[-] (X) ;
            \node[obs, right=of V] (Y) {$Y$} edge[-] (V) ;
            \node[obs, right=of Y] (B) {$B$} edge[-] (Y) ;
			
		\end{tikzpicture}
		\caption{$A$ and $B$ are \emph{coupled over} $(X, Y, \emptyset)$ given $\emptyset$ (as defined in \cref{def:s_coupling}) but neither $A$ and $B$ nor $X$ and $Y$ are \emph{coupled} given $\emptyset$ (as defined in \cref{def:coupling}) as they are not adjacent.}
		\label{fig:coupling_vs_s_coupling}
	\end{figure}
\end{example}
Since coupling can be a sufficient and necessary condition in certain cases, as \cref{cor:coupling_necessary_and_sufficient} shows, it is natural to ask whether \cref{prop:sufficient_cond_non_graphoid} is also necessary in this setting.
The following example shows that this is not the case.
\begin{example}[coupling over nodes is not necessary]
    \label{ex:s-coupling_not_necessary}
    Suppose we have an independence model that is Markovian and faithful to the graph $G$ in \cref{fig:s_coupling_not_necessary}.
     If we estimate a graph using the procedure in \cref{prop:protocol_graph_markovian} with the ordering $\pi=(W, X, Y, Z)$ and tests $L_\pi$, we will arrive at the graph $\hat G$. 
     Now $G$ entails $X\ind Y\mid \{W, Z\}$ while $\hat G$ entails the opposite.
     So clearly, this statement is PGM-redundant with respect to $L_\pi$.
     Yet, $X$ and $Y$ are coupled over $(X, Y, \{W, Z\})$ given $W$ in $\hat G$.
     So \cref{prop:sufficient_cond_non_graphoid} does not detect this PGM-redundancy.
     
\begin{figure}[htp]
	\centering
    \subfloat[Graph $G$]{
    \begin{tikzpicture}
		\node[obs] (W) {$W$} ;
		\node[obs, below left=of W] (X) {$X$} edge[-] (W) ;
		\node[obs, below right =of W] (Y) {$Y$} edge[-] (W) ;
		\node[obs, below right=of X] (Z) {$Z$} edge[-] (X) ;
		\draw (Y) edge[-] (Z);
		
	\end{tikzpicture}
    }
    \hspace{2cm}
    \subfloat[Graph $\hat G$]{
	\begin{tikzpicture}
		\node[obs] (W) {$W$} ;
		\node[obs, below left=of W] (X) {$X$} edge[<-] (W) ;
		\node[obs, below right =of W] (Y) {$Y$} edge[<-] (W) ;
		\node[obs, below right=of X] (Z) {$Z$} edge[<-] (X) ;
		\draw (Y) edge[->] (Z);
        \draw (X) edge[->] (Y);
		
	\end{tikzpicture}
    }

	\caption{The \emph{in}dependence $X\ind Y\mid \{W, Z\}$ in $G$ is PGM-redundant given the test from \cref{prop:protocol_graph_markovian}, since $G$ and $\hat G$ differ in this statement but not in the previous tests. But $X$ and $Y$ are coupled over ($X$, $Y$, $\{W, Z\})$ given $W$, so \cref{prop:sufficient_cond_non_graphoid} cannot be applied.}
	\label{fig:s_coupling_not_necessary}
\end{figure}
\end{example}

\section{THE MMD ALGORITHM}
\label{sec:mmd}
In the following, we will propose an assumption and an algorithm that can correct violations of (a) in \cref{lem:sp} in some cases, but at the price of being even more expensive than the SP algorithm.
The main purpose of this is not to propose a practical alternative to SP, but rather to demonstrate that the utility of redundant CI-statements is not tied to a specific algorithm or assumption.
\begin{assumption}[Minimum Markov Distance]
\label{ass:mmd}
	Let $\cG$ be a set of graphical models, $G^*\in\cG$, and $M$ an independence model. 
	We say $G^*$ and $M$ fulfil the \emph{minimum Markov distance assumption} (MMD) iff
	\begin{displaymath}
		G^* \in \argmin_{G\in\cG} \MD(G, M).
	\end{displaymath}
\end{assumption}
Not surprising, assumption \ref{ass:mmd} is weaker than  faithfulness as \cref{prop:faithfulness_implies_mmd} shows.
\todo{Split senctences and put between examples?}
\begin{proposition}
\label{prop:faithfulness_implies_mmd}
\todo{Introduce $\cG$ again?}
	Let $G^*\in\cG$ be a graphical model, and $M$ be an independence model that is Markovian and faithful to $G^*$.
	Then $G^*$ and $M$ also fulfil MMD.
	Further, there are independence models $M$ that fulfil MMD 
    relative to $G$
    but are not Markovian and faithful to any graph $G\in \cG$.
\end{proposition}
\begin{proof}[Proof of \cref{prop:faithfulness_implies_mmd}]
    Let $G^*\in\cG$ be a graphical model and $M$ be an independence model that is Markovian and faithful to $G^*$.
    Then by definition we have $M = M_{G^*}$.
    This means $\MD(M, G^*) = 0$ and since the Markov distance is non-negative, surely $\MD(M, G^*) \le \MD(M, G)$ for any DAG $G$.
    To see that there are cases where MMD holds but $M$ is not faithful to any DAG, consider an independence model $M$ that is Markovian and faithful to the graph $G$ in \cref{fig:sp_fails_against_mmd} except for the \emph{in}dependences $X_1\ind X_2\mid X_4$ and $X_2\ind X_3\mid X_1$.
    For faithfulness, a graph $G$ would have to contain no edge between $X_1$ and $X_2$. But then the dependence $X_1\not\ind X_2$ would violate the Markov property.
    Therefore, there is no DAG that is Markovian and faithful to the given independence model, yet MMD holds.

    				\begin{figure}[htp]
		\centering
		\begin{tikzpicture}
			\node[obs] (X1) {$X_1$} ;
			\node[obs, right=of X1] (X2) {$X_2$} edge[<-] (X1) ;
			\node[obs, right=of X2] (X3) {$X_3$} edge[<-] (X2) ;
            \node[obs, right=of X3] (X4) {$X_4$} edge[<-] (X3) ;
			\draw (X1) edge[->, bend left=20] (X4);
		\end{tikzpicture}
		\caption{If this graph is the ground truth but we have the unfaithful \emph{in}dependences $X_1\ind X_2\mid X_4$ and $X_2\ind X_3\mid X_1$,
        there is no Graph that is Markovian and faithful to the observed independence model. Further, MMD is fulfilled but SP outputs a different graph.}
		\label{fig:sp_fails_against_mmd}
	\end{figure}
\end{proof}
Perhaps more interesting, \cref{ex:mmd_holds_sp_not} shows a case where SP fails but MMD still recovers the correct graph.
\begin{example}[MMD holds but not SP]
\label{ex:mmd_holds_sp_not}
    Consider again an independence model $M$ that is Markovian and faithful to the graph $G$ in \cref{fig:sp_fails_against_mmd} except for the \emph{in}dependences $X_1\ind X_2\mid X_4$ and $X_2\ind X_3\mid X_1$.
    It can be verified easily that $G$ and $M$ fulfil MMD with $
    \MD(G, M) = 2$.
    But SP cannot recover $G$, as the algorithm would output $G - (X_2\to X_3)$.
	\label{ex:sp_fails_against_mmd}

\end{example}
Yet, MMD is not strictly weaker than the SP algorithm, as \cref{ex:sp_works_mmd_not} demonstrates.
\begin{example}[SP works but not MMD]
\label{ex:sp_works_mmd_not}
	Consider the graph $G$ with nodes $X, Y, Z$ but no edges.
    Assume further that due to false positive CI-tests, we get the \emph{in}dependence model 
    $M$ with $Y\not\ind Z,\, X\not\ind Y\ |\  Z$,\, $Y\not\ind Z\ |\  X$ 
    and \emph{in}dependence otherwise (note that this model is not Graphoid).
    Clearly $\MD(G, M) = 3$. But for the graph $G'$ with the additional edge $Y\to Z$, we have $\MD(G', M) = 2$, so $G$ and $M$ do not fulfil MMD.
    Yet,  SP would return $G$.
\end{example}

\subsection{MMD Experiments}

We generated a multivariate Gaussian distribution that is Markovian to a spanning tree over five nodes as described below. 
We then used the algorithm from \cref{prop:error_correction_trees} (MMD algorithm) and a simplified version of PC to recover the tree, which we will call \emph{TreePC}. 
As we can see in \cref{fig:tree_discovery}, MMD can indeed profit from considering additional PGM-redundant tests.

For the data generation, we first pick a random spanning tree with five nodes.
To this end, we initialise a matrix with uniformly distributed numbers from $[0, 1)$, interpret it as an adjacency matrix of a weighted graph, and find a maximum weight spanning tree using Kruskal's algorithm.
We then pick a node as root uniformly at random and orient all edges away from the root in a depth-first search to get a Markov equivalent DAG.
Then we can recursively draw samples from this graph:
we uniformly pick coefficients for a linear structural causal model from $(-1, -0.1] \cup [0.1, 1)$ and draw noise from a standard normal distribution.
For each dataset, we generate 500 samples.

To find the underlying tree under the MMD assumption, we conduct all CI-tests in the set $S$ from \cref{prop:error_correction_trees} using the Fisher $Z$ test from \texttt{causal-learn} with $\alpha$-threshold $0.01$.
We calculate the Markov distance with respect to $S$ for all possible spanning trees over the nodes.

As a baseline, we consider a simplified version of the PC algorithm.
This way, we only use tests of the same conditioning set size and can rule out that the difference is due to the statistical condition of the problem.
First, we skip the initial phase, where PC would conduct all marginal independence tests, as in a spanning tree, all nodes are dependent anyway.
We then proceed with the tests with a single conditioning variable as usual.
Since, in the limit of infinite data, these tests are already sufficient to identify the graph, we do not consider larger conditioning sets.
Further, we stop once the current graph is a tree, as any further CI-tests could only violate the spanning property.

For each of the resulting graphs, we calculate the structural Hamming distance \citep{tsamardinos2006max}, i.e., the number of differing edges, to the ground truth graph.
We repeat the experiment for 1000 datasets.

Finally, we conducted a Mann-Whitney $U$ test for the null-hypothesis that the distributions of the SHD of MMD and TreePC are not stochastically ordered. 
The test yields a $p$-value of $p=6.54\cdot 10^{-103}$.
\begin{figure}
    \centering
    \includegraphics[width=.4\linewidth]{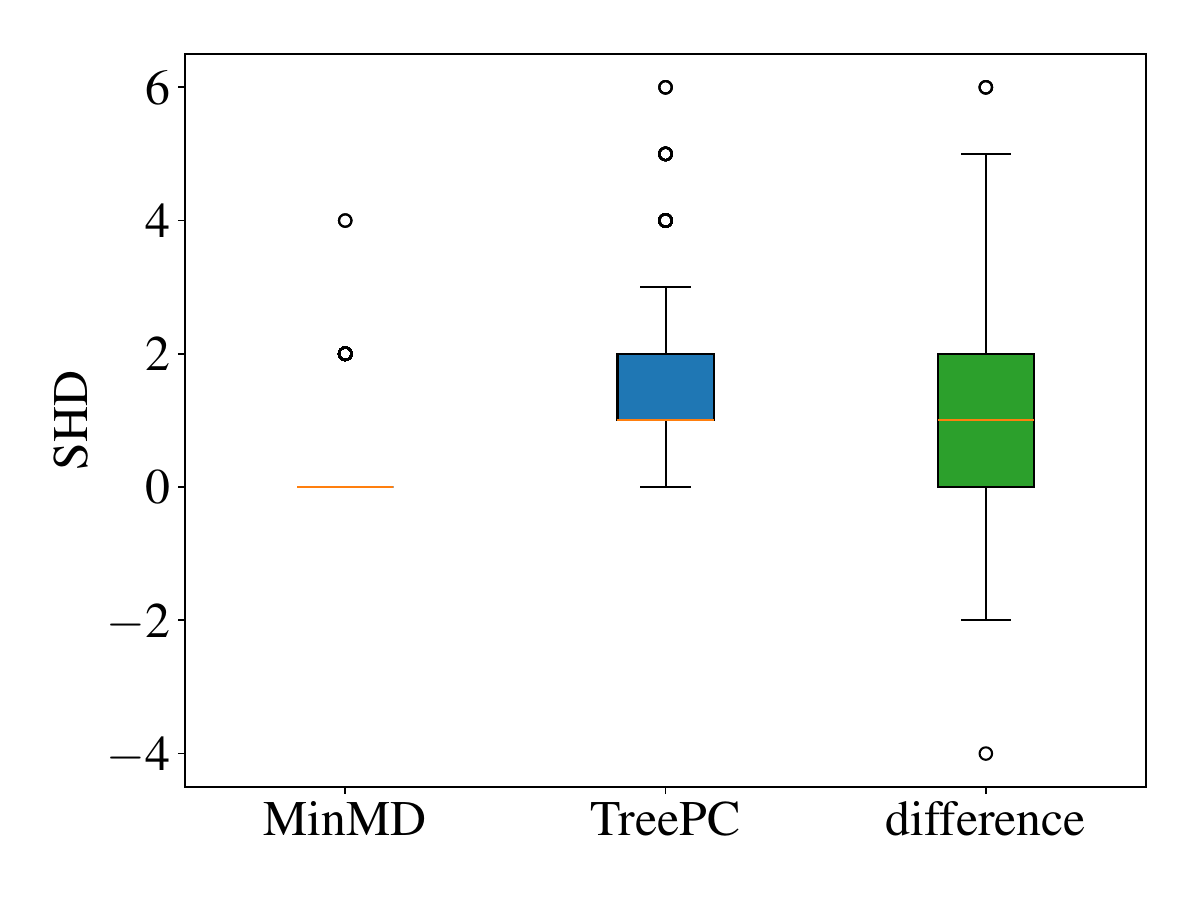}
    \caption{Structural Hamming distance for simplified PC algorithm and MMD algorithm from \cref{prop:error_correction_trees}. MMD outperforms PC.}
    \
	\label{fig:tree_discovery}
\end{figure}

As another check to demonstrate the utility of the PGM-redundant tests, we have also conducted an additional experiment, similar to \cref{fig:sp_fails_against_mmd}.
We compared three different sets of tests that we could use in the optimization. 
One set is the baseline, one contains additional GA-redundant tests, and the third one contains additional non-GA-redundant tests. 
We observed that the additional GA-redundant tests did not change the results, while the additional non-GA-redundant tests significantly improved the results.

More formally, we start with the TreePC algorithm that we used as a baseline before. Denote the set of tests conducted by TreePC with $S'$, the set of tests from MinMD as in \cref{prop:error_correction_trees} with $S$, and the empirical independence model as $M$. 
Note that $S'\subseteq S$. To find GA-redundant CI-statements, we applied axiom 5 of the Graphoid axioms (Intersection) to all combinations of tests in $S$. 
This way, we get GA-redundant marginal independence statements that we denote with $U$ (this is not a comprehensive list of GA-redundant statements). 
Note that the tests in $U$ should statistically be better conditioned than the tests in $S$. 
For the derivation of the statements in $U$, we did not use all tests in $S$. 
Denote with $W$ the subset of $S$ of CI-statements that we used to derive $U$. 
As a baseline, we took the spanning tree with minimal Markov distance on $S'\cup W$, i.e., $\min_{T\in\cT} \MD_{S'\cup W}(T, M)$.
Call this algorithm $A_1$. 
The rationale for including $S'$ is that we have at least a set of tests that is sufficient to identify the graph in the limit of infinite data. 
Then we include the GA-redundant tests and optimize $\min_{T\in\cT} \MD_{S'\cup U \cup W}(T, M)$. 
Let this be $A_2$. 
We finally compare this with the MinMD version from the original experiment (which optimizes over $S$) and call it $A_3$.
This way, we have one algorithm that optimizes over $S'\cup W$ ($A_1$), one that uses the same tests plus GA-redundant tests ($A_2$), and one that uses additional non-GA-redundant tests ($A_3$).

In our experiment, visualised in \cref{fig:tree_discovery_plus_graphoid}, we see that $A_1$
outperforms TreePC already, as it is restricted to the correct model class and also has some PGM-redundant tests. 
We further see that $A_1$ and $A_2$ do not differ at all (in the considered experiments). 
This is what we would expect, as the GA-redundant tests contain no additional information. 
Finally, $A_3$ performs significantly (with respect to the Mann-Whitney $U$ test) better than $A_1$ and $A_2$ with a $p$-value of  $p=2.49\cdot 10^{-12}$. 

The fraction of GA-redundant tests ranges between 0 and 41.9\%, with 15.5\% on average.
Under faithfulness and in the limit of infinite data, we should not be able to apply axiom 5, i.e., we are only able to derive GA-redundant statements with axiom 5 when there are at least two tests with insufficient power.
Thus, we also tried the same experiment with a sample size of 100 to increase the fraction of GA-redundant tests.  
Indeed, we can find more GA-redundant tests like this, namely a fraction of 29.1\% on average.
But we still get qualitatively the same result with $p$-value of $p=4.18\cdot 10^{-7}$. 

\begin{figure}
    \centering
    \includegraphics[width=.4\linewidth]{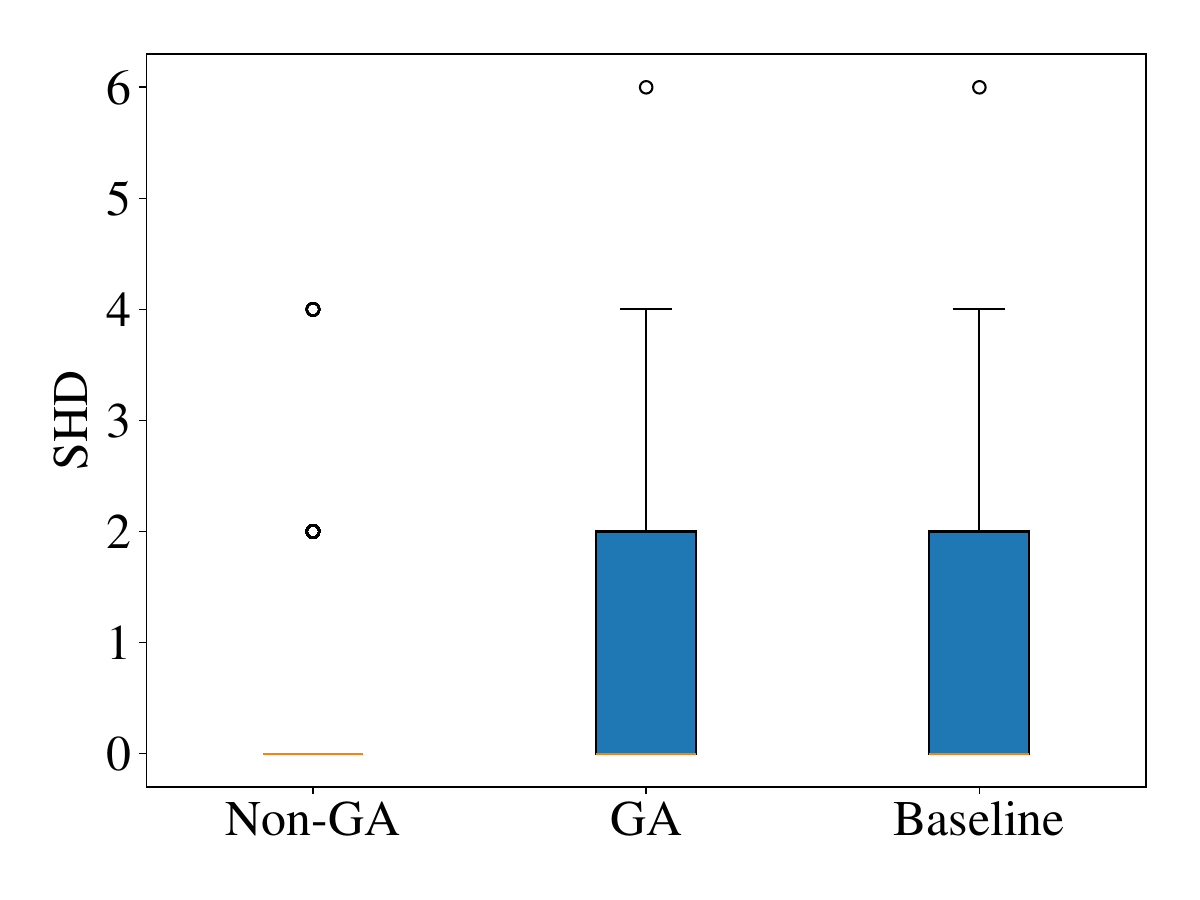}
    \caption{Structural Hamming distance when we modify MMD to optimize over different sets of tests. Additional GA-redundant tests do not change the result. Additional non-GA-redundant tests improve the results significantly.}
    \
	\label{fig:tree_discovery_plus_graphoid}
\end{figure}

\section{ADDITIONAL EXPERIMENTS}
\label{sec:graphoid_vs_graphical_synthetic}
In the experiment in \cref{fig:graphoid_vs_graphical_sachs}, we have seen that the PGM-redundant CI-tests seem to indicate more wrong predictions than the GA-redundant tests.
In the following experiments, we want to investigate how the redundant tests behave in other scenarios where we would expect them to indicate errors and scenarios where the learned graphs are mostly correct.
To this end, we generate a DAG with five nodes according to an Erdös-Rényi model with edge probability $0.3$.
We then uniformly draw coefficients for a linear structural equation model from $(-1, -0.1] \cup [0.1, 1)$.
For each variable, we add Gaussian noise with zero mean and unit variance.
For \cref{fig:graphoid_vs_graphical_small} we generate 50 samples, while for \cref{fig:graphoid_vs_graphical_large} we generate 2000 samples.

We learn DAGs using the procedure from \cref{prop:protocol_graph_markovian} with the Fisher $Z$ test from \texttt{causal-learn} and a topological ordering from the ground truth graph.
We then find PGM-redundant CI-statements via \cref{prop:sufficient_cond_non_graphoid} and \cref{cor:iterated_sufficient_redundancy} and GA-redundant CI-statements via \cref{prop:protocol_graph_markovian} and \cref{prop:coupling}.
After conducting a test, we add it to the set of previously conducted tests.
We report the fraction of wrong predictions by the graphical model, where we ignore examples where we did not find any PGM-redundant or GA-redundant CI-statements, respectively.
This reduces the effective sample size to 364 samples in \cref{fig:graphoid_vs_graphical_small} and 750 in \cref{fig:graphoid_vs_graphical_large}.

The small dataset is supposed to yield worse graph discovery results and therefore is expected to show plenty of wrongly predicted CI-statements, and indeed we observe that the correct ground truth graph is only recovered for  $21.8\%$ of the learned graphs, while with the larger sample we recover the correct graph in $98\%$ of the cases.
And indeed we see that in the latter case the difference between the tests is considerably smaller, as we were hoping.
Finally, we conducted two Mann-Whitney $U$ tests for the null-hypotheses that the distributions of the test errors are not stochastically ordered.
For \cref{fig:graphoid_vs_graphical_small} we get a $p$-value of $p=8.94\cdot 10^{-36}$.
For \cref{fig:graphoid_vs_graphical_large} we get $p=2.55\cdot 10^{-72}$, which is also clearly significant.
We hypothesise that mainly the extreme values of the distribution are responsible for the latter $p$-value, as the empirical median is identical, unlike in \cref{fig:graphoid_vs_graphical_small}.
In conclusion, we think that these experiments demonstrate that the additional errors indicated by the PGM-redundant CI-tests in \cref{fig:graphoid_vs_graphical_sachs} cannot be explained by, e.g., the PGM-redundant tests being more error-prone by nature.
This corroborates our hypothesis that they are a promising tool to evaluate graph discovery.

\begin{figure}[htp]
	\centering
    \subfloat[Dataset with 50 samples.]{
	   \includegraphics[width=.4\linewidth]{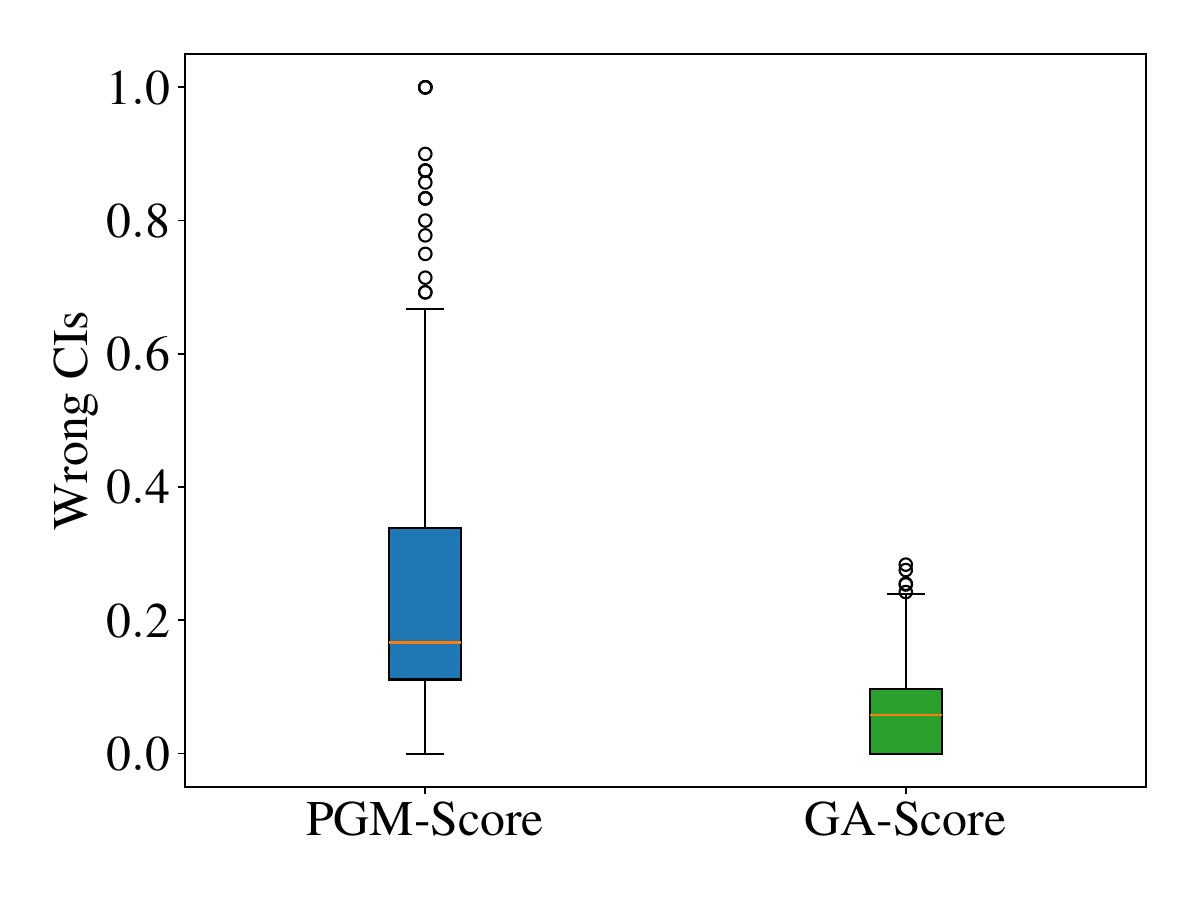}
       \label{fig:graphoid_vs_graphical_small}
    }\hspace{1cm}
    \subfloat[Dataset with 2000 samples.]{
    \includegraphics[width=.4\linewidth]{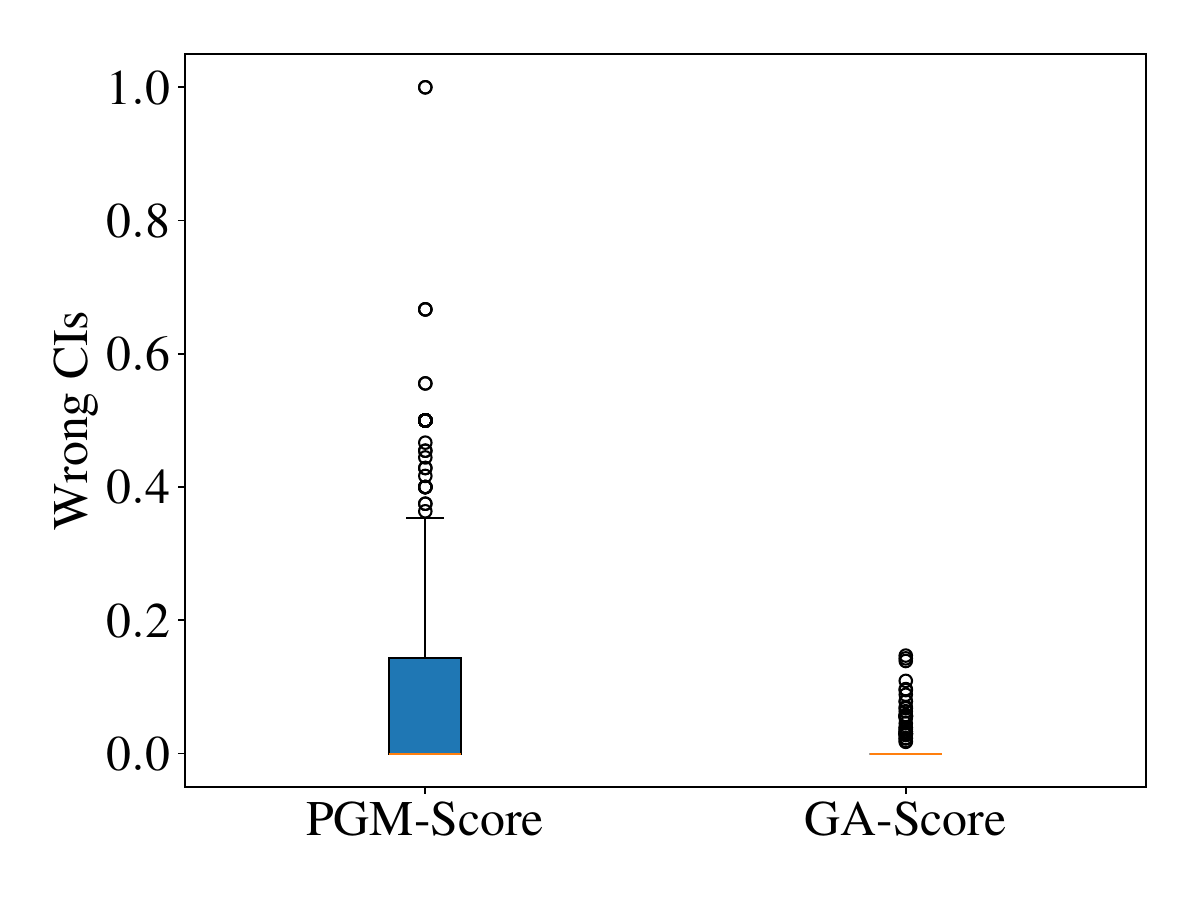}
    \label{fig:graphoid_vs_graphical_large}
    }
	\caption{Incorrect predictions of PGM- versus GA-redundant CI-statements on synthetic datasets with five nodes and different sample sizes.}
	\label{fig:graphoid_vs_graphical}
\end{figure}

\section{EXPERIMENTAL DETAILS}
\label{sec:experiment_details}

\paragraph{\Cref{fig:p_values_redundant}} For the experiment in \cref{fig:p_values_redundant}, we first generate a DAG with four nodes according to an Erdös-Rényi model with edge probability $1/2$.
We then uniformly draw coefficients for a linear structural equation model \citep{pearl2009causality} from $(-1, -0.1] \cup [0.1, 1)$.
For each variable, we add Gaussian noise with zero mean and unit variance.
We generate 2000 samples for each dataset.
Then, we randomly permute a list containing all triplets in $\CI(\bV)$, where $\bV$ is the set of nodes.
For every triplet $(X, Y, \bZ)$ in this list, we conduct a Fisher $Z$ test for independence, implemented in \texttt{causal-learn} \citep{zheng2024causal}.
We also check with the \texttt{Z3} SMT solver \citep{z3} whether the result of $\CI(X, Y\mid \bZ)$ follows from the previously conducted tests (with $\alpha$-threshold $0.01$) via Graphoid axioms.
If so, we add the $p$-value of the test to either the list of implied dependences or \emph{in}dependences, respectively.
We repeat the same experiment 16 times and keep adding the $p$-values to the same lists.
The plot in \cref{fig:p_values_redundant} shows these two lists of $p$-values and $\alpha$.
\Cref{fig:correlation_p_values_fraction} shows how many of the tests were found to be GA-redundant.

\begin{figure}[htp]
	\centering
    \includegraphics[width=.35\linewidth]{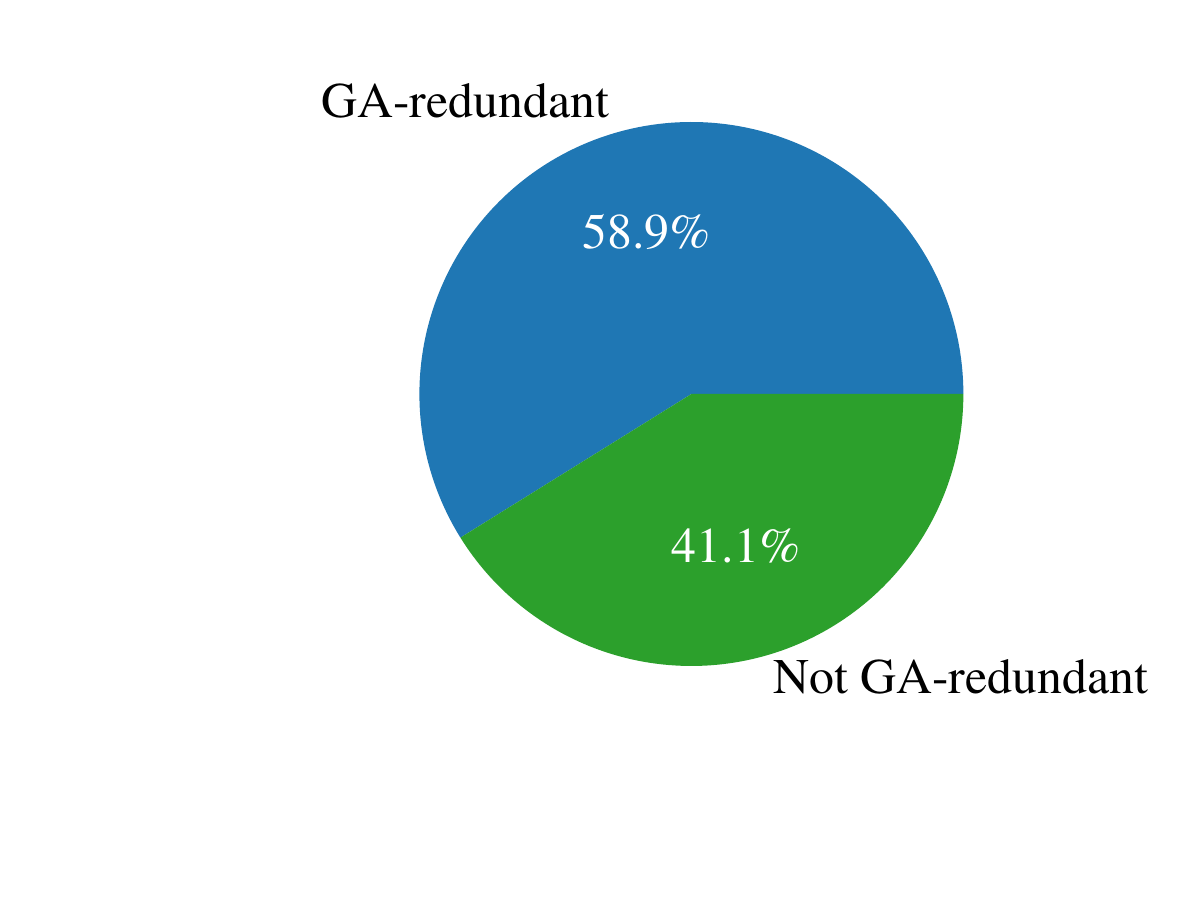}
	\caption{Fraction of all tests that were identified as GA-redundant in the experiment in \cref{fig:p_values_redundant}.}
	\label{fig:correlation_p_values_fraction}
\end{figure}

\paragraph{\Cref{fig:graphically_redundant_two_datasets}} For the experiment in \cref{fig:graphically_redundant_two_datasets}, we generate two kinds of datasets, one generated by a DAG and the other by an undirected graphical model.
The ground truth graphs can be seen in \cref{fig:ground_truth_two_datasets}.
All variables are binary.
For the DAG, we draw coefficients for the conditional distributions of each variable given all possible values of its parents.
We then recursively sample the values along the topological ordering of the graph, starting from $W$.
We drew $P(X=1 \ |\  W=0),\, P(Y=1 \ | \  W=0),\, P(Z=1\ | \  X=0, Y=0)$ and $P(Z=1\ |\  X=1, Y=1)$ uniformly from $[0.3, 0.7)$ and then set
\begin{align*}
    P(X=1 \mid W=1) &= 1-P(X=1 \mid W=0)\\
    P(Y=1 \mid W=1) &= 1-P(Y=1 \mid W=0)\\
    P(Z=1\mid X=0, Y=1) &= 1- P(Z=1\mid X=0, Y=0)\\
    P(Z=1\mid X=1, Y=0) &= 1- P(Z=1\mid X=1, Y=1).
\end{align*}
For the undirected model, we used the \texttt{pgmpy} package \citep{ankan2024pgmpy}.
For each edge $X-Y$ in the graph, we add a factor $\phi$ to a factor graph model, where we pick the value $\phi(X=0, Y=0)$ and $\phi(X=1, Y=1)$ uniformly from $[0.1, 0.3)$ and set 
\begin{align*}
\phi(X=0, Y=1) &= 1 - \phi(X=0, Y=0) \\
\phi(X=1, Y=0) &= 1- \phi(X=1, Y=1).
\end{align*}
We then draw the distribution using Gibbs sampling with the default parameters of \texttt{pgmpy}.
Note that neither of the graphs' independence models is Markovian and faithful to the other.
For each dataset we generated 2000 samples.

We then use the methods described in \cref{prop:markov_network_markovian,prop:protocol_graph_markovian} to construct an undirected graphical model and a DAG on each dataset, where we used the $\chi^2$-test implemented in \texttt{causal-learn} with $\alpha$-threshold $0.01$.

Finally, we identify GA-redundant tests using \cref{prop:coupling,prop:markov_network_markovian,prop:protocol_graph_markovian}.
We exclude these tests and report the fraction of remaining unused CI-tests, where the graphical implication contradicts the empirical test result.
There were no cases where all statements were detected as GA-redundant CI-statements.

In most practical settings it would be more relevant to use the PGM-tests detected by \cref{prop:sufficient_cond_non_graphoid}.
Yet, in this toy scenario, we have already seen in \cref{ex:s-coupling_not_necessary} that the only test that would reveal the graph misspecification between the undirected ground truth and a learned DAG would not be detected by \cref{prop:sufficient_cond_non_graphoid} (although it is PGM-redundant).
The experiment nonetheless shows that we pick up the relevant signal if we exclude all detected GA-redundant tests.

We repeated the experiment 1000 times.
Additionally, we conduct Mann-Whitney $U$ tests for the null-hypotheses that the distributions of errors of the DAGs and the undirected graphs are not stochastically ordered. 
For the datasets generated by DAGs, we get $p=3.03\cdot 10^{-23}$ and for the datasets from the undirected model we get a p-value below machine precision.

\begin{figure}[htp]
	\centering
    \subfloat[Ground truth for the DAG datasets in \cref{fig:graphically_redundant_two_datasets}.]{
	\begin{tikzpicture}
		\node[obs] (W) {$W$} ;
		\node[obs, below left=of W] (X) {$X$} edge[<-] (W) ;
		\node[obs, below right =of W] (Y) {$Y$} edge[<-] (W) ;
		\node[obs, below right=of X] (Z) {$Z$} edge[<-] (X) ;
		\draw (Y) edge[->] (Z);
		
	\end{tikzpicture}
    }\hspace{2cm}
    \subfloat[Ground truth for the undirected graph datasets in \cref{fig:graphically_redundant_two_datasets}.]{
    \begin{tikzpicture}
		\node[obs] (W) {$W$} ;
		\node[obs, below left=of W] (X) {$X$} edge[-] (W) ;
		\node[obs, below right =of W] (Y) {$Y$} edge[-] (W) ;
		\node[obs, below right=of X] (Z) {$Z$} edge[-] (X) ;
		\draw (Y) edge[-] (Z);
		
	\end{tikzpicture}
    }
	\caption{Ground truth graphs for the experiment in \cref{fig:graphically_redundant_two_datasets}.}
	\label{fig:ground_truth_two_datasets}
\end{figure}

\Cref{fig:redundant_tests_fractions} shows the fraction of the tests that were identified as GA-redundant by the criterion from \cref{prop:coupling,prop:protocol_graph_markovian,prop:markov_network_markovian} among $\CI(\bV) \setminus L$ and $\CI(\bV) \setminus L_\pi$, respectively.
In other words, the fraction of tests that follow from the tests used in   \cref{prop:protocol_graph_markovian,prop:markov_network_markovian} among all others. 

\begin{figure}[htp]
	\centering
	   \includegraphics[width=.4\linewidth]{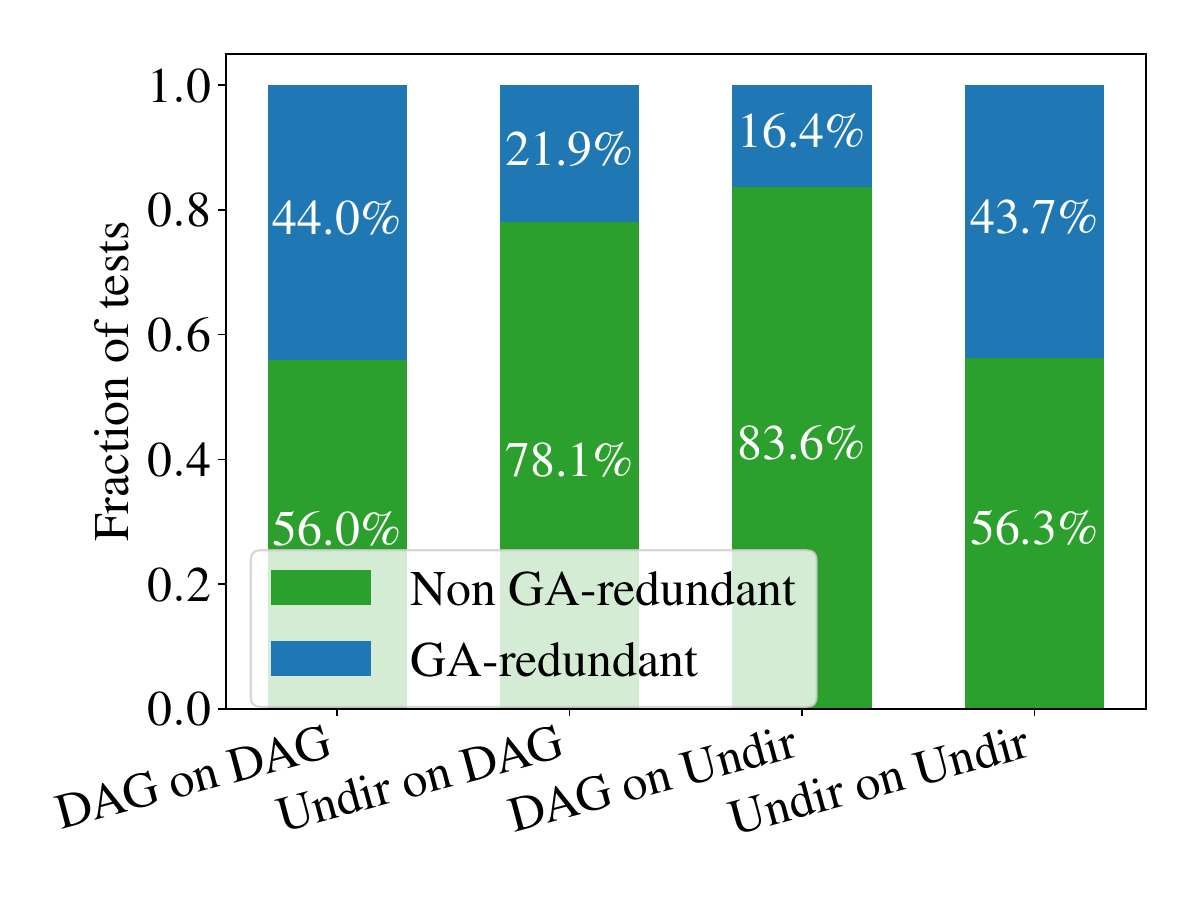}
	\caption{Fraction of all unused tests that were identified as GA-redundant in the experiments in \cref{fig:graphically_redundant_two_datasets}.}
	\label{fig:redundant_tests_fractions}
\end{figure}

\paragraph{\Cref{fig:graphoid_vs_graphical_sachs}} For the experiment in \cref{fig:graphoid_vs_graphical_sachs}, we use the dataset from \citet{sachs2005causal} as provided in the \texttt{causal-learn} package and the ground truth graph given in the original paper.
We repeat the experiment 1000 times.
For each run, we draw a causally sufficient subset (i.e., a subset such that according to the ground truth graph, there is no hidden confounder) of five variables using rejection sampling, and we draw a bootstrap sample of the same size as the original dataset.
We then apply the algorithm from \cref{prop:markov_network_markovian} using a topological ordering of the ground truth graph and the Fisher $Z$ test to get a DAG.
As the significance threshold of the tests, we used $\alpha = 0.001$.
In no instance the expected ground truth was recovered.
From this DAG, we derive PGM-redundant CI-statements via \cref{prop:sufficient_cond_non_graphoid} and \cref{cor:iterated_sufficient_redundancy}.
We conduct these tests and report the fraction of CI-tests, where the graphical implication contradicts the empirical test result. 
After conducting a test, we add it to the set of previously conducted tests.
Finally, we conducted a Mann-Whitney $U$ test for the null-hypothesis that the distribution of errors is not stochastically ordered. We found a $p$-value of $p=2.92\cdot 10^{-144}$.

\Cref{fig:graphoid_vs_graphical_fractions} shows the fraction of the tests that were identified as PGM-redundant with respect to $L_\pi$ and following tests by the criterion from \cref{prop:sufficient_cond_non_graphoid,cor:iterated_sufficient_redundancy} among $\CI(\bV) \setminus L_\pi$, with $L_\pi$ as in \cref{prop:protocol_graph_markovian}.

\begin{figure}[htp]
	\centering
    \subfloat[Sachs dataset.]{
	   \includegraphics[width=.3\linewidth]{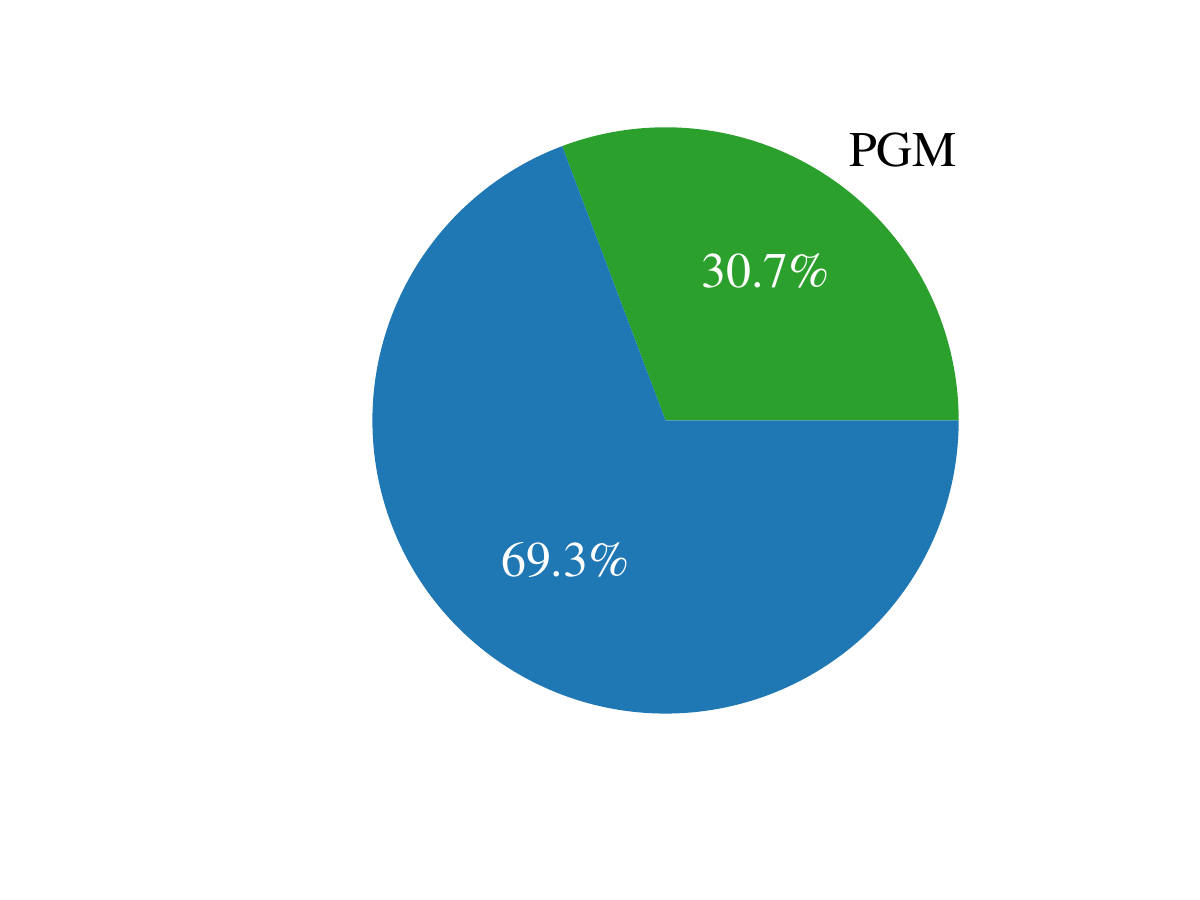}
       \label{fig:fractions_sachs}
    }
    \subfloat[Dataset with 50 samples.]{
	   \includegraphics[width=.3\linewidth]{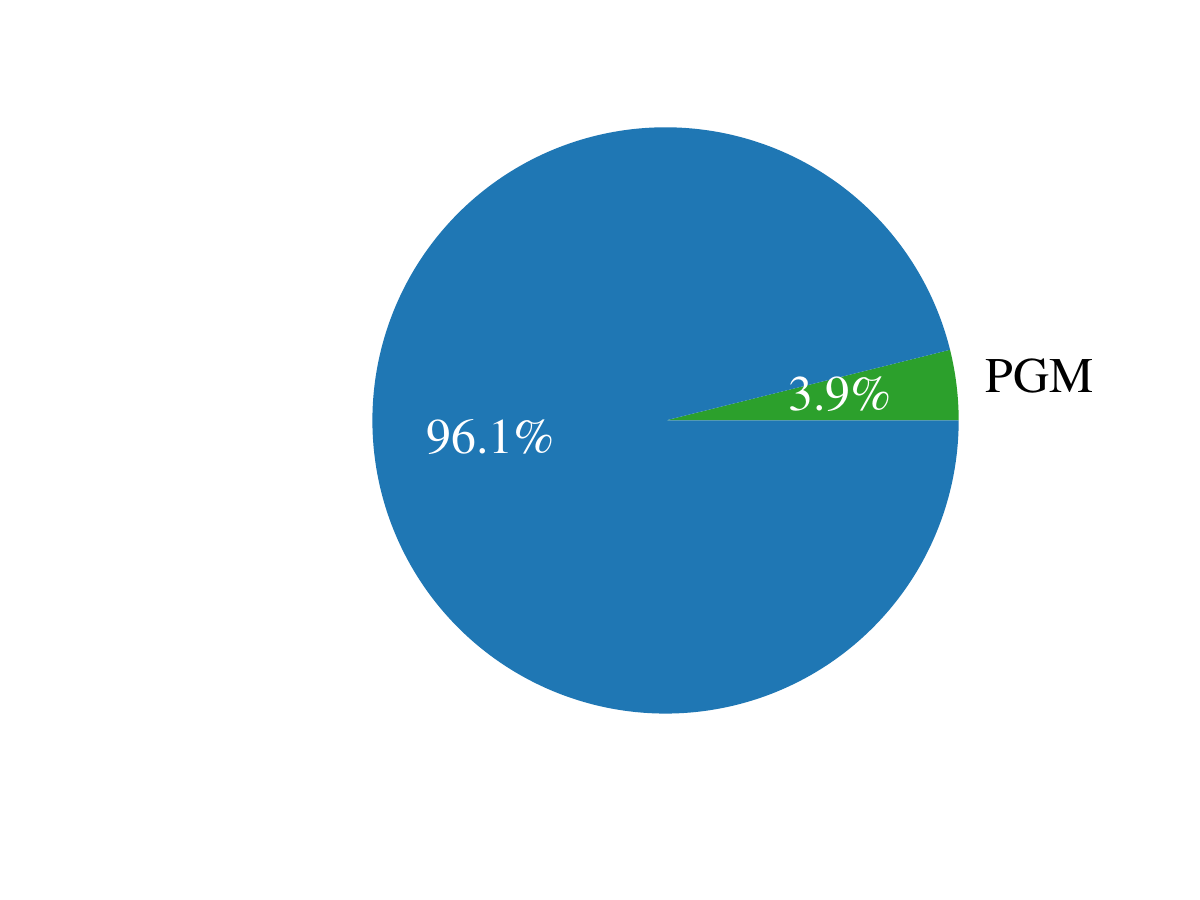}
       \label{fig:fractions_synthetic_small}
    }
    \subfloat[Dataset with 2000 samples.]{
    \includegraphics[width=.3\linewidth]{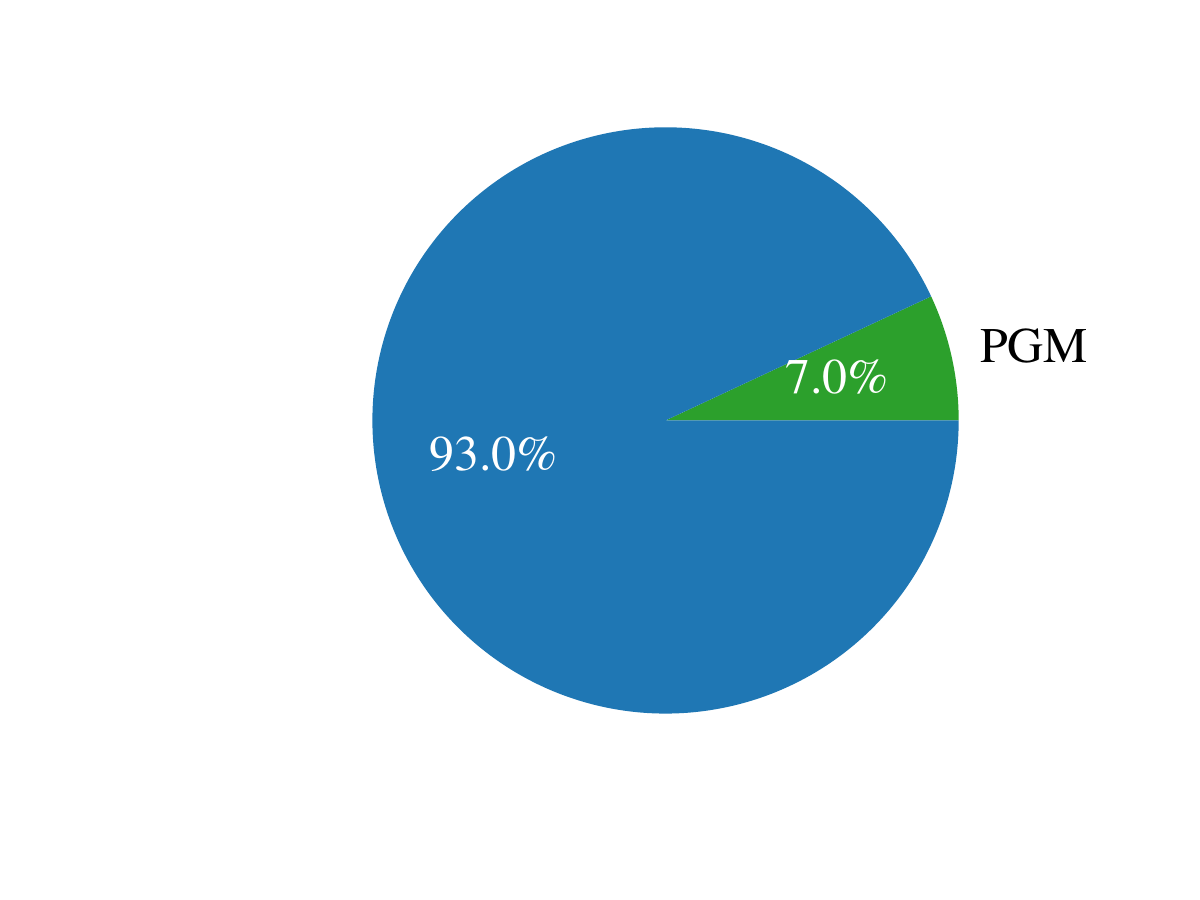}
    \label{fig:fractions_synthetic_large_5}
    }
	\caption{Fraction of all unused tests that were identified as PGM-redundant in the experiments in \cref{fig:graphoid_vs_graphical_sachs,fig:graphoid_vs_graphical}.}
	\label{fig:graphoid_vs_graphical_fractions}
\end{figure}

All experiments, including the ones in \cref{sec:mmd,sec:graphoid_vs_graphical_synthetic}, were run on an Apple M3 Pro processor with 18 GB RAM. The experiment for \cref{fig:p_values_redundant} took roughly a day, 
while all other experiments finished in a couple of minutes.

\end{document}